\documentclass{article}
\usepackage{amsmath,amsfonts,bm}

\def\eqref#1{equation~\ref{#1}}
\def\1{\bm{1}}

\def\rvg{{\mathbf{g}}}

\def\rvx{{\mathbf{x}}}
\def\rvy{{\mathbf{y}}}

\DeclareMathAlphabet{\mathsfit}{\encodingdefault}{\sfdefault}{m}{sl}
\SetMathAlphabet{\mathsfit}{bold}{\encodingdefault}{\sfdefault}{bx}{n}

\newcommand{\ptrain}{\hat{p}_{\rm{data}}}

\usepackage{microtype}
\usepackage{graphicx}
\usepackage{subcaption}
\usepackage{booktabs} % for professional tables
\usepackage{hyperref}
\usepackage{xcolor}

\usepackage[accepted]{icml2021}
\usepackage{float}
\usepackage{adjustbox}
\usepackage{multirow}
\usepackage{amsmath, amssymb, amsthm, enumerate, framed}
\usepackage[ruled,lined]{algorithm2e}
\usepackage{wrapfig}
\usepackage[normalem]{ulem}
\makeatletter
\newcommand{\removelatexerror}{\let\@latex@error\@gobble}
\makeatother

\usepackage{thmtools,thm-restate}

\theoremstyle{plain}
\newcounter{theoremcounter}
\newtheorem{lemma}[theoremcounter]{Lemma}

\theoremstyle{definition}

\def\ptrue{p_{\text{ideal}}}
\def\ptrain{p_{\text{train}}}

\begin{document}

\twocolumn[
\icmltitle{Examining and Combating Spurious Features under Distribution Shift
}

% It is OKAY to include author information, even for blind
% submissions: the style file will automatically remove it for you
% unless you've provided the [accepted] option to the icml2021
% package.

% List of affiliations: The first argument should be a (short)
% identifier you will use later to specify author affiliations
% Academic affiliations should list Department, University, City, Region, Country
% Industry affiliations should list Company, City, Region, Country

% You can specify symbols, otherwise they are numbered in order.
% Ideally, you should not use this facility. Affiliations will be numbered
% in order of appearance and this is the preferred way.
\icmlsetsymbol{equal}{*}

\begin{icmlauthorlist}
\icmlauthor{Chunting Zhou}{cmu}
\icmlauthor{Xuezhe Ma}{usc}
\icmlauthor{Paul Michel}{cmu}
\icmlauthor{Graham Neubig}{cmu}
\end{icmlauthorlist}

\icmlaffiliation{cmu}{Language Technologies Institute, Carnegie Mellon University, Pittsburgh, USA}
\icmlaffiliation{usc}{Information Sciences Institute, University of Southern California, Log Angeles, USA}

\icmlcorrespondingauthor{Chunting Zhou}{chuntinz@cs.cmu.edu}

% You may provide any keywords that you
% find helpful for describing your paper; these are used to populate
% the "keywords" metadata in the PDF but will not be shown in the document
\icmlkeywords{Machine Learning, distributionally robust optimization, group robustness, noisy groups, ICML}

\vskip 0.3in
]

% this must go after the closing bracket ] following \twocolumn[ ...

% This command actually creates the footnote in the first column
% listing the affiliations and the copyright notice.
% The command takes one argument, which is text to display at the start of the footnote.
% The \icmlEqualContribution command is standard text for equal contribution.
% Remove it (just {}) if you do not need this facility.

\printAffiliationsAndNotice{}  % leave blank if no need to mention equal contribution
% \printAffiliationsAndNotice{\icmlEqualContribution} % otherwise use the standard text.

\begin{abstract}
A central goal of machine learning is to learn robust representations that capture the causal relationship between inputs features and output labels.
% While machine learning models are able to learn c omplex prediction rules by minimizing the training error, they also 
However, minimizing empirical risk over finite or biased datasets often results in models latching on to \emph{spurious correlations} between the training input/output pairs that are not fundamental to the problem at hand.
% Models that fit these correlations often fail on inputs where the spurious correlation does not hold.
In this paper, we define and analyze robust and spurious representations using the information-theoretic concept of \emph{minimal sufficient statistics}.
We prove that even when there is only bias of the input distribution (i.e.~\emph{covariate shift}), models can still pick up spurious features from their training data.
Group distributionally robust optimization (DRO) provides an effective tool to alleviate covariate shift by minimizing the \emph{worst-case} training loss over a set of pre-defined groups.
Inspired by our analysis, we demonstrate that group DRO can fail when groups do not directly account for various spurious correlations that occur in the data.
% under ``imperfect'' partitions where groups are not created from exact spurious factors.
To address this, we further propose to minimize the worst-case losses over a more flexible set of distributions that are defined on the \emph{joint distribution} of groups and instances, instead of treating each group as a whole at optimization time.
Through extensive experiments on one image and two language tasks, we show that our model is significantly more robust than comparable baselines under various partitions.
Our code is available at \url{https://github.com/violet-zct/group-conditional-DRO}.

\end{abstract}
\vspace{-5mm}

\section{Introduction}
Many machine learning models that minimize the average training loss via empirical risk minimization (ERM) are trained and evaluated on randomly shuffled and split training and test sets.
However, such in-distribution learning setups can hide critical issues: models that achieve high accuracy on average often underperform when the test distribution drifts away from the training one~\citep{hashimoto2018fairness,koenecke2020racial,koh2020wilds}.
Such models are often ``right for the wrong reasons" due to reliance on \emph{spurious correlations} (or ``\emph{dataset biases}")~\citep{torralba2011unbiased,goyal2017making,mccoy2019right,gururangan2018annotation}, heuristics that hold for most training examples but are not inherent to the task of interest, such as strong associations between the presence of green pastures background with the label ``cows'' in image classification.
%such as strong associations between the presence of negation word ``not" with the label ``negation" in natural language inference (NLI) 
Naturally, models that use such features will fail when tested on data where the correlation does not hold.
% \cz{do we need to add that bad influence of models with high worst-group errors on social applications.}

\begin{figure}[t]
\vspace{-3mm}
    \centering
      \includegraphics[width=0.4\textwidth]{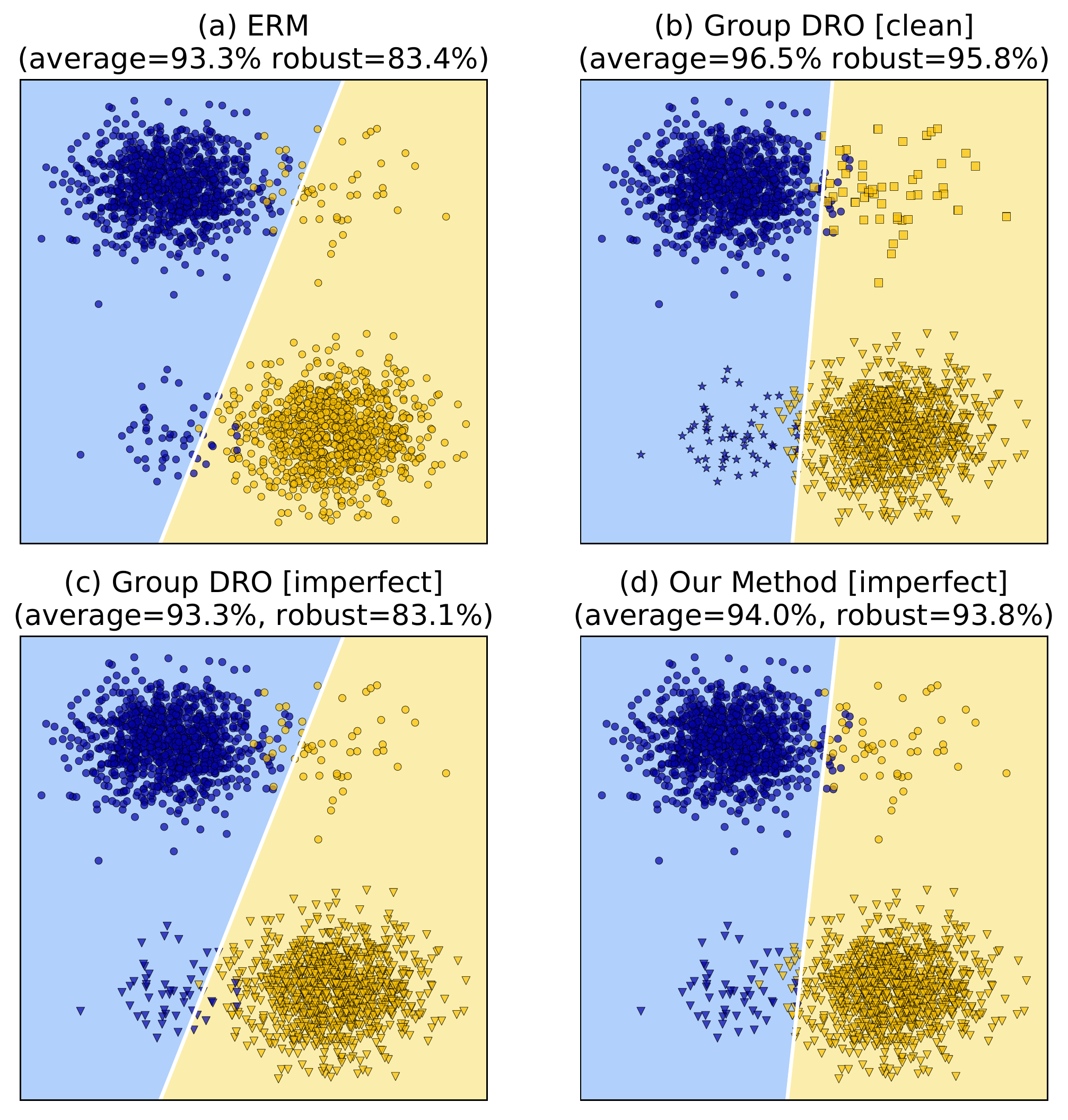}
    \vspace{-3mm}
    \caption{Consider data points $x$ in $\mathbb{R}^2$ with two classes $y$. The vertical axis of $x$ is a spurious feature that highly correlates with $y$, and the horizontal axis is the robust feature. There are two subclasses in each class, where the top-right and lower-left are two minority subclasses. The robust accuracy is test worst-case accuracy over the four subclasses. 
    We train a linear classifier with different methods. 
    For models trained with the clean partitions, each subclass is a group.
    For the imperfect partitions, dots with the same shape is a group (best viewed in color).
    % (b) Model trained with group DRO on clean partitions (each subclass as a group) obtains average accuracy 96.5\% and robust accuracy 95.8\%. (c) Model trained with group DRO on imperfect partitions (dots in the same shape as a group) obtains average accuracy 93.3\% and robust accuracy 83.1\%. (d) Our method trained on imperfect partitions obtains average accuracy 94.0\% and robust accuracy 93.8\% (best viewed in color).
    }
    \label{fig:intro}
    \vspace{-6mm}
\end{figure}
% Although it is not fully clear \gn{``it is not fully clear'' seems a bit like too much of a hedge to me. It's pretty clear: ML models fit correlations, and if the correlations in a dataset are spurious they will fit to them. It's pretty simple.}
% \cz{In my view, this is not super clear. Since there are both spurious correlations and robust correlations in the dataset, but why does the model uses the spurious correlations (actually spurious features that makes the classifier to build spurious correlations between these spurious features and labels) to predict labels instead of the robust ones? The following literature provides a view of it.}
% \gn{Hmm, I'm not sure what your confusion is? This seems really simple to me: if there is no inductive bias to prefer robust features over spurious ones then the model has no way of distinguishing which features are spurious, and which are robust, and thus it will use whatever features it has available to it opportunistically. This is particularly true if the features are frequent (and thus not regularized out) and have low conditional entropy of the label given the feature (and thus reduce loss by a large amount). This seems like machine learning 101, but maybe I'm missing something?}
Recent work has investigated how models trained with ERM learn spurious features that do not generalize, from the points of view of
% recent works have investigated this phenomenon 
causality~\cite{arjovsky2019invariant}, understanding model overparameterization~\citep{sagawa2020investigation} and information theory~\citep{lovering2020info}. 
However, these works have not characterized the idea of spurious features mathematically.
% However, these spurious features have not been characterized in a mathematically rigorous way.
% \gn{This is more clear than before, but it's a bit of a strong claim. You're essentially calling previous work in machine learning ``not mathematically rigorous'', which some of the writers of the previous papers may take offense to. It might be better to be a bit more specific.}. \cz{I changed the previous sentence slightly, also add this one.}
In this paper, we characterize spurious features from an information-theoretic perspective.
We consider prediction of target random variable $Y\in \mathcal{Y}$ from input variable $X \in \mathcal{X}$ and characterize spurious features learned under changes to the input distribution $p(X)$ (i.e.~covariate shift).

A central goal of machine learning is to learn true causal relationships between $X$ and $Y$ in a manner robust to spurious factors concerning the variables.
We assume that there exists an \emph{``ideal''} data distribution $\ptrue$ (short for $\ptrue(X, Y)$ below) which contains data from all possible experimental conditions concerning the confounders that cause spurious correlations, both observable and hypothetical~\citep{lewis2013counterfactuals,arjovsky2019invariant,bellot2020generalization}. 
For example, consider the problem of classifying images of cows and camels~\citep{beery2018recognition}. 
Under the ideal conditions, we assume that pictures of cows and camels on any background can be collected, including cows in deserts and camels in green pastures.
Therefore, under $\ptrue$ the background of the image $X$ is no longer a spurious factor of the label $Y$.
However, such an ``ideal" distribution $\ptrue$ is not accessible in practice~\citep{bahng2020learning,koh2020wilds,mccoy2019right}, 
% such as selection biases and domain shifts, 
and our training distribution $\ptrain$ (often, in practice, 
an associated empirical distribution) does not match $\ptrue$. 
ERM-based learning algorithms indiscriminately fit all correlations found in $\ptrain$, including spurious correlations based on confounders~\citep{tenenbaum2018building,lopez2016dependence}.

To investigate the spurious features learned under the distribution shift from $\ptrue$ to $\ptrain$, we first characterize those features of $X$ which most efficiently capture all possible information needed to predict $Y$. We define these \emph{robust features} using the notion of \emph{minimal sufficient statistic (MSS)}~\citep{dynkin2000necessary,cvitkovic2019minimal} under $\ptrue$.
% the MSS learned on $P_{true}(X, Y)$ does not contain unnecessary spurious features.
We then examine \textit{whether the features learned under $\ptrain$ contain spurious features compared to the MSS learned under $\ptrue$}.
Through our analysis, we find that even only with covariate shift, the features learned on $\ptrain$ can contain spurious features or miss robust features of $\ptrue$.

% Although the ultimate goal is to learn a model that is robust under $P_{true}$, at test time we often divide the test set into $m=|\mathcal{S}| \times |\mathcal{Y}|$ groups based on the combinations of values of some known spurious factor $S$ and label $Y$ and assess the \emph{worst-case} (i.e., \emph{robust}) performance over all the groups in the test set.
Models that fit spurious correlations in $\ptrain$ can be vulnerable to groups (subpopulations of $\ptrue$/$p_{\text{test}}$) where the correlation does not hold.
A common approach to avoid learning a model that suffers high worst group errors
% error in the worst of these groups
% due to relying on spurious correlations in $P_{train}$, 
is group distributionally robust optimization (group DRO), a training procedure that efficiently minimizes the worst expected loss over a set of groups in the training data~\citep{oren2019distributionally,sagawa2019distributionally}.
% Group DRO is an instance of distributionally robust optimization (DRO)~\citep{ben2013robust,duchi2016statistics} that optimizes for the worse-case loss over a family of distributions $\mathcal{Q}$ (the uncertainty set). \cz{do we need to remove the above sentence to save space?} \gn{That's fine I think.}
The partition of groups can be defined in several ways, such as by 
% demographic attributes~\citep{michel2021modeling}, 
presence of manually identified potentially spurious features~\citep{sagawa2019distributionally}, data domains~\citep{koh2020wilds}, or topics of text~\citep{oren2019distributionally}. 
In a typical setup, the groups of interest in the test set align with those used to partition the training data. 
Under such setups, group DRO usually outperforms ERM with respect to the worst-group accuracy.
We contend that this is because it promotes learning robust features that perform uniformly well across all groups.
%and reduce the spurious features due to covariate shift. 
% GDRO minimizes losses on the worst-case distribution from all the mixtures of subpoplulation that reshapes the biased training distribution which contains spurious correlations, alleviating covariate shift in $P_{train}$.
However, in many tasks, we can not collect clean group membership of training examples due to expensive annotation cost or privacy concerns regarding e.g. demographic identities of users or other sensitive information.

Inspired by our analysis of spurious features, we demonstrate that group DRO can fail under ``imperfect" partitions of training data that are not consistent with the test set, especially when 
reducing spurious correlation in one group could exacerbate the spurious correlations in another
%each training group contains subgroups of data from different spurious factor values 
(\S\ref{sec:gdro:fail}), as shown in Fig.~\ref{fig:intro}.
This is because group DRO treats each training group as a unit, preventing it from adjusting learning weights differently for subgroups within each group.
% thus could not alleviate the covariate shift of each group with $P_{true}$
Recent work has proposed to use sophisticated unsupervised clustering algorithm to search for meaningful subclasses~\citep{sohoni2020no} and execute group DRO on the found subclasses. 
To learn robust models under noisy protected groups, \citet{wang2020robust} designs robust approaches that is based on an estimate of a noise model between the clean and noisy groups.
Instead of relying on good partitions of groups or a not readily available noise model, we propose group-conditional DRO (GC-DRO) that defines the uncertainty set over the joint distribution of groups and their instances (i.e. $q(G)q(X, Y|G)$). 
Every training example is reweighted by both its group weight and the instance-level weight, which offers a more flexible uncertainty set compared to group DRO.
Through extensive experiments on three tasks --- facial attribute classification, natural language inference, and toxicity detection, we show that GC-DRO significantly outperforms both ERM and group DRO in various partitions of training data and demonstrate the robustness of GC-DRO against various group partitions.

\vspace{-2mm}
\section{Preliminaries on Robust Representations}
\label{sec:def:spurious}
To study spurious features, we need to formally define which features or properties of the data describe spurious correlations, and which features are robust features relevant to the task at hand. 
In supervised learning we are interested in finding a good representation $T(X)$ of the input $X$\footnote{We assume that $T(X)$ is a deterministic mapping of $X$ given neural network parameters.} that is useful to predict a target label $Y$.
What characterizes the optimal representations of $X$ w.r.t. $Y$ is much debated, but a common assertion is that $T(X)$ should be a \emph{minimal sufficient statistic} (MSS) of $X$ for $Y$~\citep{adragni2009sufficient,shwartz2017opening,achille2018emergence,cvitkovic2019minimal}, which is:

(i) $T(X)$ should be \emph{sufficient} for $Y$, i.e. $\forall x\in \mathcal{X}, t\in \mathcal{T}, y\in \mathcal{Y}, p(x|t, y)=p(x|t)$, which is equivalent to $p(y|t, x)=p(y|t)$. This means given the value of $T(X)$, the distribution of $X$ does not depend on the value of $Y$.

(ii) Given that $T(X)$ is sufficient, it should be \emph{minimal} w.r.t.~$X$, i.e.~for any sufficient statistic $S$, there exists a deterministic function $f$ such that $T = f(S)$ almost everywhere w.r.t.~$X$. This means for any measurable, non-invertible function $g$, $g(T)$ is no longer sufficient for $Y$.%~\citep{cvitkovic2019minimal}.

In other words, the minimal sufficient statistics most efficiently capture all information useful for predicting $Y$.
The notion of MSS has been connected to Shannon's information theory~\citep{kullback1951information,cover1999elements} and extended to any joint distribution $P(X, Y)$ of $X$ and $Y$ in the information bottleneck (IB) framework~\citep{tishby2000information,shamir2010learning,kolchinsky2018caveats}, which provides a principled way to characterize the extraction of relevant information from $X$ for predicting $Y$.
% \cz{The IB framework provides a general method to extract relevant features from X, so I think it's okay to directly use the verb?} \gn{it's not an algorithm that actually extracts features though, right? it's just a mathematical framework that describes how features are extracted or the properties of the features that are extracted?}
Loosely speaking, learning a MSS $T$ is equivalent to maximizing $I(T(X); Y)$ (sufficiency) and minimizing $I(X; T(X))$ (minimality). 

\vspace{-1mm}
\noindent \textbf{Robust Features.} 
Suppose $\mathcal{A}$ contains all possible combinations of spurious variables, both observable and hypothetical, and we consider datasets $\mathcal{D}_{(a,y)} = \{x_i\}_{i=1}^{N_{a, y}}$ collected under each condition of $(a\in\mathcal{A}, y\in\mathcal{Y})$, where each $\mathcal{D}_{(a, y)}$ contains examples that are \emph{i.i.d.} according to some probability distribution $p(X|y, a)$. 
We define $\ptrue$ as the mixture distribution of $p(X|y, a)$ with uniform weights over $(a, y) \in \mathcal{A} \times \mathcal{Y}$. 
% For example, in the image classification example, suppose light $\mathcal{A}_1$ and background $\mathcal{A}_2$ are all spurious variables for the task, e.g.~$\mathcal{S}_1$ contains $\{$dark, normal, bright$\}$, $\mathcal{A}_2$ contains $\{$grassland, desert, water$\}$, etc. 
% Then under the combination of each label and each combination of $(y, a_1, a_2)$, we collect data $\mathcal{D}^{(a_1, a_2, y)}$ and form $\ptrue$ with uniform weights over each condition.
% \cz{Graham: Does this definition look good? Do we need the examples again? I feel only with the abstract definition is enough, because we have stated a similar example without multiple spurious variables in the intro.}
% \gn{The example is useful, but not essential, for me. I'd also be OK with an abstract definition as long as you define it somewhat formally and for multiple spurious features like you do here.}
% Recall that we assume $\ptrue(X, Y)$ contains all possible experimental conditions concerning the confounding factors, both observable and hypothetical.
Thus, MSS learned on $\ptrue$ provide a good candidate for \emph{robust features} $T(X)$ (sometimes denoted $T_{ideal}(X)$ for clarity), which most efficiently capture the information from $X$ necessary for predicting $Y$ on a distribution that is free of spurious factors.
% To tie back to the concept of spurious and robust features, MSS learned on $\ptrue$ are \emph{not spurious} because they are \emph{minimal}, and thus they are can be considered \emph{robust features} $T(X)$ that most efficiently capture the information from $X$ necessary for predicting $Y$ \gn{I tried to reword this a little to explain that they are not spurious because they are minimal, not sure if you feel this is better than the old (commented out) sentence}. 

\vspace{-1mm}
\noindent \textbf{Spurious Features.} In contrast, we define representations $T'(X)$ that contain spurious features.
% on top of the robust features $T(X)$ learned from $\ptrue$ (also denoted $T_{ideal}(X)$ \gn{why not $T_{robust}(X)$? best to be consistent with the naming.} \cz{I use the superscript to emphasize which distribution the MSS is learned from, my later proofs uses this to distinguish that with $T_{train}(X)$.} \gn{Maybe that mention of $\ptrue$ makes it more consistent?} for clarity).
Specifically, the entropy of $T'(X)$ conditioned on $T(X)$ under $\ptrue$ is positive.
\vspace{-1mm}
\begin{equation}
\small
\label{eq:spu}
    H_{\textcolor{blue}{ideal}}(T'(X) | T(X)) > 0
\vspace{-1mm}
\end{equation}
Because these learned features are not deterministic given $T(X)$ then they contain \emph{additional} information that is not useful for predicting $Y$.%
\footnote{Note that it is not just the case of $T'(X)$ containing redundant features, in which case $H(T'(X) | T(X)) = 0$.}
% We consider such features as spurious ones since they can not be deterministically represented by the robust features $T(X)$. 
For example, in image classification, knowing that the image contains a horse, we cannot predict the background with certainty (a horse could be on a race track or a beach).
Another example in natural language inference (NLI) task is that model learned on a biased data set often associates negation with the label ``contradiction". This is another spurious feature under our definition, because given the meaning of a sentence (robust features), whether it contains negation or not is not deterministic, e.g.~``Don't worry.'' and ``Be calm.'' are synonymous but only one contains negation.
A classifier that uses these spurious features can suffer from the risk of learning the \textit{spurious correlations} between $T'(X)$ and the labels $Y$.

\vspace{-2mm}
\section{Spurious Features under Covariate Shift}
\label{sec:inv:cov}
\vspace{-1mm}
The training data is often marred by various abnormalities, such as selection biases~\citep{buolamwini2018gender} and confounding factors~\citep{gururangan2018annotation}. 
We ask \emph{if the MSS learned under $\ptrain$ are robust features under $\ptrue$}.
Note that we do not study how to learn MSS via ERM in this paper, on the other hand, considering that MSS provides a good candidate for robust representations, we want to study if the MSS learned under $\ptrain$ contains spurious features with respect to the MSS learned under $\ptrue$, which are universal robust features against various spurious factors.

We consider the distribution shift in $p(X)$,\footnote{It is often assumed that $p(Y|X)$ is invariant in supervised learning problems~\citep{arjovsky2019invariant}.} also known as covariate shift~\citep{david2010impossibility}, and we show that the entropy of MSS learned under $\ptrain$ conditioned on the robust features is zero in Theorem~\ref{thm:cov} with proofs in \S\ref{append:covariate:proof}.
\begin{restatable}{theorem}{cov}
\label{thm:cov}
Suppose that there is only covariate shift in $\ptrain$, i.e.~$\exists x \in \mathcal{X}_{train} \text{~s.t.~}\ptrain(x) \ne \ptrue(x)$ but $\ptrain(Y|X=x) = \ptrue(Y|X=x), \,\, \forall x \in \mathcal{X}_{train}$. Let $T_{\text{train}}(X)$ be the MSS representation learned under $\ptrain$, then we have: 
\vspace{-2mm}
\begin{equation}\label{eq:mss}
\small
H_{\textcolor{blue}{train}}(T_{train}(x)|T_{ideal}(x)) = 0.
\vspace{-2mm}
\end{equation}
\end{restatable}
% \paul{Is this an equivalence? like if we can learn robust features then there is only covaraite shift (alternatively if there is concept shift then we can't learn the robust features)? This could be useful to reuse in the section on dro}
% \cz{I think I didn't prove anything about concept shift, and I am not able to because $P(Y|X)$ are not the same. As I noted in the footnote, people usually think there is no concept shift in supervised learning}
Theorem~\ref{thm:cov} tells us that $T_{train}(X)$ is deterministic given $T_{ideal}(X)$ under $\textcolor{blue}{\ptrain}$ (shown in blue to distinguish from Eq.~\ref{eq:spu}).
However, this does not imply $H_{ideal}(T_{train}(X) | T_{ideal}(X)) = 0$ under $\textcolor{blue}{\ptrue}$.
Thus, we \emph{cannot} conclude that $T_{train}(X)$ contains no spurious features. 
We further discuss the implications with two cases based on the relationship between the support of input $\mathcal{X}_{train}$ and that of $\mathcal{X}_{ideal}$: (1) $\mathcal{X}_{train} = \mathcal{X}_{ideal}$ and (2) $\mathcal{X}_{train} \subset \mathcal{X}_{ideal}$.
When the input support of $\ptrain$ is equal to that of $\ptrue$, we have the following corollary:
\begin{restatable}{corollary}{cor}
\label{cor}
Suppose $\mathcal{X}_{train}$ =  $\mathcal{X}_{ideal}$ in Theorem~\ref{thm:cov}, then $T_{train}(X)$ is also the MSS under $\ptrue$. 
\end{restatable}
\vspace{-2mm}
Corollary~\ref{cor} corroborates the findings in \citet{wen2014robust} that the (unweighted) solution learned by ERM is also the robust solution when only covariate shift exists and $\mathcal{X}_{train}=\mathcal{X}_{ideal}$. 
In practice, however, this assumption does not hold (because we only have datasets with limited support) and thus the representation $T_{train}(X)$ learned by ERM is not necessarily equivalent to $T_{ideal}(X)$.
By Theorem~\ref{thm:cov}, $T_{train}(X)$ is deterministic given $T_{ideal}(X)$ under $\ptrain$, which implies that the information contained in $T_{train}(X)$ is equal to or less than that contained in $T_{ideal}(X)$. In the former case, $T_{train}(X)$ can be equivalent in representation to $T_{ideal}(X)$ but can also contain spurious features that co-occur with the robust features in the training data. In the latter case, $T_{train}(X)$ can miss robust features in $T_{ideal}(X)$. We demonstrate these two cases with synthetic experiments in Appendix~\ref{app:syn} due to space limit.

\vspace{-1mm}
\noindent \textbf{Discussion.} We have discussed the cases of learning spurious features when the model learns MSS under $\ptrain$. 
However, we normally adopt maximum likelihood estimation (MLE) as an instantiation of ERM for classification problems. 
We provide the connection of MLE with learning MSS via the information bottleneck method~\citep{tishby2000information,shamir2010learning} in the Appendix~\ref{app:connect}, where under certain assumptions, we can view MLE as an objective that approximately learns MSS. 
% In addition, we present synthetic experiments in~\ref{app:syn} to confirm our theoretical analysis.

\section{Does Group DRO Learn Robust Features?}
The discussions in \S\ref{sec:inv:cov} suggest that under covariate shift, directly learning from the empirical data distribution
can result in learning the spurious correlations satisfied by the majority of the training data. 
When the spurious factors are known, we can apply group distributionally robust optimization (group DRO), which reweights the losses of different groups associated with spurious factors to alleviate covariate shift and learn robust features that generalize to both minority and majority groups. In this section, we first review group DRO and discuss under which cases it can fail.

\subsection{Group Distributionally Robust Optimization}
\label{sec:gdro}
Group DRO is an instance of distributionally robust optimization~\citep{ben2013robust,duchi2016statistics} that minimizes the worst expected loss over a set of potential test distributions $\mathcal{Q}$ (the uncertainty set):
\vspace{-2mm}
\begin{equation}
\small
    \mathcal{L}_{\mathrm{DRO}}(\theta) = \sup_{q\in\mathcal{Q}}\mathbb{E}_{(x, y)\sim q}\big[ \ell(x, y; \theta) \big ]
    \label{eq:dro}
\vspace{-2mm}
\end{equation}
This worst-case objective upper bounds the test risk for all $q_{\mathrm{test}} \in \mathcal{Q}$, which is useful for learning under train-test distribution shift.
However, its success crucially depends on choosing an adequate uncertainty set that encodes the possible test distributions of interest.
Choosing a general family of distribution as the uncertainty set, such as a divergence ball around the training distribution~\citep{ben2013robust,hu2013kullback,gao2016distributionally}, encompasses a wide set of distribution shifts, but can also lead to a conservative objective emphasizing implausible worst-case distributions~\citep{duchi2019distributionally,oren2019distributionally}.

To construct a viable uncertainty set, one can optimize models over all meaningful subpopulations or groups $g$ depending on the available source information regarding the data, such as domains, demographics, topics, etc. 
Group DRO~\citep{hu2018does,oren2019distributionally} leverages such structural information and constructs the uncertainty set as any mixture of these groups. 
Following \citet{oren2019distributionally}, we adopt the conditional value at risk (CVaR) which is a type of distributionally robust risk to achieve low losses on all $\alpha$-fraction subpopulations~\citep{rockafellar2000optimization}
% \gn{\citet{oren2019distributionally} define $\alpha$-fraction sub-populations, at least implicitly, in the following sentence after the introduce the term.} 
of the training distribution (i.e.~$\{p: \alpha p(x) \leq \ptrain(x), \forall x\}$).
As we assume that each data point comes from some group $p(x, y|g)$ and $\ptrain$ is a mixture of $m$ groups $\ptrain(g)$, we can extend the definition of CVaR to groups and construct the uncertainty set $\mathcal{Q}$ as all group distributions that are $\alpha$-\emph{covered} by $\ptrain(g)$ (or \textit{topic CVaR} \citep{oren2019distributionally}):
\vspace{-1mm}
\begin{equation}
\small
    \mathcal{Q} = \left\{q:  q(g) \leq \frac{\ptrain(g)}{\alpha}~~\forall g\right\}
\vspace{-1mm}
\end{equation}
This upper bounds the group distribution within the uncertainty set by its corresponding training distribution.
The group DRO objective then minimizes the expected loss under the worst-case group distribution:
\begin{equation}
\small
    \mathcal{L}_{\mathrm{GDRO}} = \sup_{q \in \mathcal{Q}}\mathbb{E}_{g\sim q}\mathbb{E}_{(x, y)\sim p(x, y|g)}\left[ \ell(x, y; \theta) \right]
    \label{eq:gdro}
\vspace{-1mm}
\end{equation}
Intuitively, this objective encourages uniform losses across different groups, which allows us to learn a model that is robust to group shifts. We adopt the efficient online greedy algorithm developed in~\citet{oren2019distributionally} to update the model parameters $\theta$ and the worst-case distribution $q$ in an interleaved manner. 
The greedy algorithm roughly amounts to upweighting the sample losses by $\frac{1}{\alpha}$ which belong to the $\alpha$-fraction of groups that have the worst losses. We present the detailed algorithm in Appendix~\ref{app:alg:gdro}.

\subsection{Group DRO Can Fail with Imperfect Partitions}
\label{sec:gdro:fail}
As discussed earlier, we aim to learn a model that is robust to spurious factors. 
For example, in toxicity detection, a robust model should perform equally well on data from different demographic groups.
% \cz{Through discussions, we find it necessary to connect with the theories in Section 3. After revisions, I made the connection by examining why group DRO can succeed or fail under different partitions: whether $\mathcal{Q}$ contains $\ptrue$. Does this connection make sense?}
% \gn{The following two sentences seem a bit repetitive with the previous subsection and previous section.} 
% \cz{By repetitive, do you think they are redundant or need to be paraphrased? I changed to the following, is it good?}
% However, as discussed in \S\ref{sec:inv:cov}, ERM models can pick up spurious features or miss robust features under covariate shift.
% To mitigate covariate shift, group DRO minimizes the worst-case loss under the uncertainty set $\mathcal{Q}$, consisting of mixtures of sub-group distributions. 
Group DRO mitigates covariate shift by minimizing the worst-case loss under the uncertainty set $\mathcal{Q}$, consisting of mixtures of sub-group distributions.
Intuitively, given that optimizing $\ptrue$ allows for learning of robust, non-spurious features, defining a $\mathcal{Q}$ that covers $\ptrue$ is highly advantageous from a learning perspective.

If we know all the spurious attributes of the training data  $\mathcal{A}$, we can adopt the setup in \citet{sagawa2019distributionally} that divides the data into $|\mathcal{A}| \times |\mathcal{Y}|$ groups, where each example belongs to one of the groups $g = (a, y)$. 
We define such grouping strategy as ``clean partitions'' in which each group is uniquely associated with one value of $(a, y)$.\footnote{Our discussions also apply to multiple spurious attributes for which the clean partition corresponds to $|\mathcal{Y}| \times \prod_i |\mathcal{A}_i|$ groups.}
If $\mathcal{A}$ contains all the spurious factors of interest, it can be seen that there exists some mixture of groups $\sum_{g=1}^m q(g) \ptrain(\cdot\mid g)$ that can recover $\ptrue$, where $q \in\Delta_m$ and $\Delta_m$ is the $(m-1)$-dimensional probability simplex.
% \paul{Alternatively, if you want to avoid defining a throwaway notation for the simplex, just say ``where $q$ is a distribution over the $m$ groups''}. 
% \gn{I agree that Paul's wording is a bit simpler. I'm not sure if there are any advantages to discussing the simplex?}
% \paul{Again in this section the notation for $p_g$ and $q$ is inconsistent. I would say use $q$ because it's less confusing than using another $p$}
% \cz{Because I used this symbol later again when discussing imperfect partitions.}
Thus, $\ptrue$ is contained in $\mathcal{Q}$.
% \gn{Given that $\ptrue$ hasn't been formally defined, it's not clear that what you've written here does actually prove that $\ptrue$ is contained in $\mathcal{Q}$?}.
Such clean partitions provide a plausible environment for group DRO to learn well in the presence of covariate shift
that causes spurious correlations in the training data. 

\begin{table}[t]
    \centering
    \small
    \begin{tabular}{l|cc|cc}
        \toprule
         \multirow{2}{*}{} & \multicolumn{2}{c|}{$G_1$} &  \multicolumn{2}{c}{$G_2$}\\
        & $S=0$ & $S=1$ & $S=0$ & $S=1$ \\
        \midrule
        $P(Y=0|S)$ & 0.5 & 0 & 1 & 0.5 \\
        $P(Y=1|S)$ & 0.5 & 1 & 0 & 0.5 \\
        \bottomrule
    \end{tabular}
    \vspace{-2mm}
    \caption{An example of imperfect partition.}
    \label{tab:gdro:example}
    \vspace{-3mm}
\end{table} 
% \cz{Paul pointed that my previous description of imperfect partition goes too fast and not clear. Since I can't provide a mathematical definition in a general form of imperfect partitions that can not recover $\ptrue$, I used concrete examples to illustrate this. Is this paragraph clear now?}
In contrast, we define ``imperfect partitions'' where each group contains samples from multiple values of $(a, y)$ such that there does not exist a $q \in \Delta_m$ that recovers $\ptrue$, in other words, $\mathcal{Q}$ does not include $\ptrue$.
In this case, group DRO can not eliminate covariate shift effectively.

To illustrate, consider a binary random variable $S \in \{0, 1\}$ following a uniform distribution, and the target label $Y \in \{0, 1\}$ also follows a uniform distribution and is independent of $S$.
Due to covariate shift, there are spurious correlations between $S=0$, $Y=0$ and between $S=1$, $Y=1$ in the training data. 
We partition the training data into two groups with an equal number of samples and the conditional distribution of $P(Y|S)$ is shown in Tab.~\ref{tab:gdro:example}. 
To prevent the model from learning the spurious correlations between $S=1$ and $Y=1$, one can upweight losses of its ``negative" samples for which the spurious correlation does not hold, i.e. samples of $(S=1, Y=0)$ in $G_2$; however, group DRO upweights the group as a whole, which inevitably also upweights the $(S=0, Y=0)$ and causes the model to latch on the spurious attribute $S=0$ to predict $Y=0$. 
Therefore, there does not exist a mixture distribution of these two groups, under which $S \perp Y$ ($\ptrue$). 
Such underlying conflicts prevent the group DRO from formulating a worst-case distribution that can eliminate covariate shift, resulting in a passive reliance on certain spurious correlations.

Imperfect partitions of training data are common in practice, as it can be expensive or infeasible to acquire the labels of spurious attributes for each training instance. 
For example, we may only have rough partitions based on the data sources or the outputs from (unsupervised) clustering algorithms.
Our analysis shows that under these practical settings, the group DRO algorithm can not effectively alleviate covariate shift due to the rigid treatment of group losses.
% \begingroup
% \vspace{-4mm}
% \removelatexerror% Nullify \@latex@error
% \resizebox{!}{0.68 \columnwidth}{
\begin{algorithm}[t]
\SetAlgoLined
\DontPrintSemicolon
\SetKwInput{Input}{Input}
\SetKwInOut{Output}{Output}
\SetCommentSty{itshape}
\SetKwComment{Comment}{$\triangleright$\ }{}
\Input{$\alpha$; $\beta$; $m$: \#groups; $n_i$: \#samples of group $i$}
Initialize historical average group losses $\hat{L}^{(0)}$, historical estimate of group probabilities $\hat{p}^{tr(0)}$, historical average instance losses $\hat{L}_g^{(0)}$ and $q^{(0)}(x, y| g)=\bm{1}^T$ for $g \in \{1,\cdots,m\}$ \\
\For{$t=1,\cdots,T$}{
    Sample a mini-batch $(\rvx, \rvy, \rvg)$ from $P_{train}$ \\
    Perform online greedy updates for $q^{(t)} (Alg.\ref{alg:gdro})$\\
    \Comment{Update model parameters $\theta$}
    $d_i = \frac{n_i q^{(t)}(\rvg_i)q^{(t)}(x, y|\rvg_i)}{\hat{p}^{train(t)}(\rvg_i)}\nabla\ell(\rvx_i, \rvy_i;\theta^{(t-1)})$\\
    $\theta^{(t)} = \theta^{(t-1)} - \frac{\eta}{|B|} \sum_{i=1}^{|B|}d_i$\\
    \If{reached inner update criterion}{
    \Comment{Update $q^{(t)}(x, y|g)$}
    \For{$g=1, \cdots, m$}{
    Sort instances in group $g$ in the decreasing order of $\ell(x, y; \theta^{t})$; denote the sorted index $\bm{\pi}^g$\\
    $\mathrm{cutoff} =\left\lceil\frac{(N-n_i)n_i\beta}{N-n_i}\right\rceil$ \\
    $q^{(t)}((x, y)_{\bm{\pi}^g(j)}|g) = \frac{1}{\beta}, \forall 1 \leq j \leq \mathrm{cutoff}$\\
    $q^{(t)}((x, y)_{\bm{\pi}^g(j)}|g) = \frac{n_i}{N}, \forall j > \mathrm{cutoff}$
    }
    %}{$q^{(t)}(x, y|g) = q^{(t-1)}(x, y|g)$
    }
}
  \caption{\label{alg:sgdro}Online greedy algorithm for GC-DRO.}
\end{algorithm}
% }
% \vspace{-3mm}
% \endgroup

\section{Proposed Method: Group-conditional DRO}
Since group DRO can be problematic with imperfect partitions, we propose a more flexible uncertainty set over the joint distribution of $(x, y, g)$, i.e.~$q(g)q(x, y|g)$, using fine-grained weights over instances within each group instead of treating the entire group as a whole. 
We extend the $\alpha$-covered distribution to both the group-level ($q(g)$) and conditional instance-level ($q(x, y|g)$) distributions to define the uncertainty set $\mathcal{Q}$.
At training time, a sample is weighted by both its group weight induced from $q(g)$ as well as the instance-level weight induced from $q(x, y|g)$.
Specifically, the new uncertainty set is
\vspace{-2mm}
\begin{equation}
\small
   \begin{split}
    \mathcal{Q}^{\alpha, \beta} &= \left\{q(g)q(x, y|g):  q(g) \leq \frac{\ptrain(g)}{\alpha}, \right.\\ 
     & \left.\frac{1}{N} \leq q(x, y|g) \leq \frac{\ptrain(x, y|g)}{\beta}, \forall x, y, g \right\},
     \label{eq:newus}
\vspace{-2mm}
\end{split} 
\vspace{-5mm}
\end{equation}
where $N$ is the number of training examples and $\alpha,\beta\in (0, 1]$.
Denote $n_i$ the number of samples in group $i$, then $p^{train}(x, y|g=i)=\frac{1}{n_{i}}$.
The second constraint of Eq.~\ref{eq:newus} can be rewritten as $\frac{1}{N} \leq q(x, y|g) \leq \frac{1}{\beta n_i}$.
Compared with the $\beta$-covered distribution, we add a lower bound $q(x,y|g)\geq \frac{1}{N}$ to compensate for imbalanced group sizes. 
With a plain $\beta$-covered distribution for $q(x, y|g)$, the DRO objective roughly upweights a $\beta$-fraction of instance losses of each group.
However, we only want to emphasize a small subset of examples that perform badly in the majority groups.
Thus, we add this lower bound to $q(x, y|g)$ in Eq.~\ref{eq:newus} to directly ``punish'' larger groups.
To see this, the percentage of examples that are upweighted by $\frac{1}{\beta}$ in group $i$ is roughly $\frac{N-n_i}{N-n_i\beta}\beta$, which is monotonically decreasing function w.r.t. $n_i$. 
%By solving $\frac{k_i}{n_i\beta} + \frac{n_i - k_i}{N} \leq 1$, we have $k_i \leq \frac{N-n_i}{N-n_i\beta}n_i\beta$, and $\frac{k_i}{n_i}$ is a monotonically decreasing function w.r.t. $n_i$. 
Therefore, the larger the group size $n_i$ is, the smaller fraction of instances in group $i$ are upweighted.

\vspace{-1mm}
\noindent \textbf{Online Optimization Algorithm.}
Similarly to the online greedy algorithm for group DRO~\citep{oren2019distributionally} (details in Appendix~\ref{app:alg:gdro}), we interleave the updates between model parameters $\theta$ and the worst-case distribution $q(g)q(x, y|g)$. 
The greedy algorithm involves sorting losses of all the variables when updating the worst-case distribution defined by the $\alpha$-covered distribution.
However, frequently updating $q(x, y|g)$ over large-scale training data (e.g millions of samples) can be costly and unstable. 
Therefore, we only update $q(g)$ at every iteration, while performing updates on $q(x, y|g)$ lazily once every epoch or when the robust accuracy on the validation set drops (inner update criterion).
We present the pseudo code for the training process in Alg.~\ref{alg:sgdro}.

\noindent \textbf{Discussions.} 
Another potential approach to circumventing the purely group-level loss is constructing an instance-level uncertainty set~\citep{ben2013robust,husain2020distributional,michel2021modeling}, however, the resulting $\mathcal{Q}$ can be too pessimistic~\citep{hu2018does,duchi2019distributionally} 
% that contains implausible worst-case distribution 
%\cz{pessimistic was discussed in 4.1}
or difficult to optimize~\citep{michel2021modeling}.
Instead, we leverage the structural information of data partitions and expand the flexibility of uncertainty set by incorporating the conditional probabilities of instances.
Furthermore, this allows us to execute the min-max optimization in an efficient manner.

\section{Experiments}
% \cz{Paul pointed that the description of imperfect partitions in the \S6.1 is not clear. For CelebA and MNLI, I created imperfect partitions manually by devising hard cases with explicit conflicts between groups. Paul pointed that ``people might complain that the imperfect partitions were only created to show off the method and they don't correspond to a tangible reality.". He also suggested a new experiment where we randomly change the attributes and still use the original clean partitions, thus we have noised groups that don't correspond to the ground-truth groups any more. However, I don't have time to finish experiments for all baselines and our method. Thus I decide to run this experiment later to prepare for the rebuttal. To dismiss the potential concerns of ``showing off our methods", I added discussions on how we motivate to design imperfect partitions below. However, it is a bit overlong. Please let me know if the following paragraph and 6.1 are clear now.}
In this section, we evaluate the proposed group-conditional DRO on one image classification task and two language tasks --- natural language inference and toxicity detection. 
To demonstrate the effectiveness of our method under various partitions of data, we first introduce the clean (group number $m=|\mathcal{A}|\times|\mathcal{Y}|$) and imperfect data partitions of each task. 
As we discussed at the end of \S\ref{sec:gdro:fail}, there are various cases where the partitions of training data are imperfect such that each group is \emph{not} purely associated with examples from one pair of $(a, y)$. 
In this section, we inspect several cases reflecting diverse properties of partitions to evaluate our method.
First, on the image and NLI tasks, we manually design adversarial partitions of data such that there are explicit conflicts between groups and purely reweighing over groups cannot eliminate covariate shift (\S\ref{sec:exp:data}). 
Second, we use the attributes provided by a supervised classifier to create the imperfect partitions of the toxicity data set (\S\ref{sec:exp:data}).
Third, we also perform unsupervised clustering on the toxicity data set to obtain imperfect partitions in \S\ref{sec:analysis}.

\vspace{-1mm}
\subsection{Data and Tasks}
\label{sec:exp:data}
\vspace{-1mm}

\begin{table}[t]
\begin{subtable}{0.5\textwidth}
    \centering
    \small
    \begin{tabular}{l|cc}
    \toprule
             & male & female \\
    \midrule
        dark &  65,487 / 1,387 & 22,880 / 48,749 \\
        blonde & 0 / 1,387 & 22,880 / 0 \\
    \bottomrule
    \end{tabular}
    \caption{The imperfect partitions for the CelebA dataset ($G_1/G_2$).}
    \label{tab:da}
    \vspace{2mm}
\end{subtable}

\begin{subtable}{0.48\textwidth}
    \centering
    \small
    \begin{tabular}{l|ccc}
    \toprule
          &  no neg & neg 1 & neg 2\\
    \midrule
        contradiction &  57,605 / 0 / 0 & 0 / 1,406 / 0 & 0 / 0 / 9,897\\
        entailment & 67,335 / 0 / 0 & 0 / 0 / 215 & 0 / 1,318 / 0 \\
        neutral & 66,401 / 0 / 0 & 0 / 0 / 251 & 0 / 1,747 / 0\\
    \bottomrule
    \end{tabular}
    \caption{The imperfect partitions for the MNLI dataset ($G_1/G_2/G_3$).}
    \label{tab:db}
    \vspace{2mm}
\end{subtable}

\begin{subtable}{0.5\textwidth}
    \centering
    \small
    \setlength{\tabcolsep}{4pt}
    \begin{tabular}{l|cccc}
    \toprule
           & White-aligned & AAE & Hispanic & Others\\
    \midrule
        abusive & 11,281 & 7,392 & 6,707 & 1,770 \\
        spam &  8,147 & 1,041 & 541 & 4,301 \\ 
        normal & 41,756 & 2,562 & 2,638 & 6,895 \\
        hateful & 2,696 & 1,420 & 509 & 340 \\
    \bottomrule
    \end{tabular}
    \caption{Statistics of each group in the clean partition of the hate speech dataset. Data of each dialect attribute (column) corresponds one group in the imperfect partition.}
    \label{tab:dc}
\end{subtable}
\vspace{-3mm}
\caption{Statistics of data in different groups partitioned by attributes (row) and labels (column).}
\vspace{-4mm}
\end{table}

\noindent \textbf{Object Recognition.} 
We use the \textbf{CelebA} dataset~\citep{liu2015deep} which has 162,770 training examples of celebrity faces. We classify the hair color from $\mathcal{Y}=\{$blond, dark$\}$ following the set up in \citet{sagawa2019distributionally}. 
In this task, labels are spuriously correlated with the demographic information --- gender of the input $\mathcal{A}=\{$female, male$\}$, which together with $\mathcal{Y}$ results in 4 clean groups.
The statistics of groups in the imperfect partition are presented in Tab.~\ref{tab:da} (separated by ``/''), each of which consists of data from multiple values of $(a, y)$. 
Concretely, we create an imperfect partition of \textbf{2} groups with two explicit spurious correlations: i) in group $G_1$ (dark, male) are spuriously correlated since we put all their counterparts (blonde, male) in group $G_2$; 
ii) similarly, (dark, female) in $G_2$ are spuriously correlated.
% having spurious correlations of different attribute values and labels that are irreconcilable. 
\begin{table*}[t]
\centering
\small
\begin{tabular}{l|l|cc|cc}
\toprule
\multirow{2}{*}{Datasets} &\multirow{2}{*}{Methods} & \multicolumn{2}{c|}{Clean Partition} &\multicolumn{2}{c}{Imperfect Partition} \\
\cmidrule{3-6}
& & Robust Acc & Average Acc & Robust Acc & Average Acc\\
\midrule
\multirow{5}{*}{Celeb-A} & ERM & 40.14 $\pm$ 0.99 & \textbf{95.92 $\pm$ 0.05} &40.14 $\pm$ 0.99 & \textbf{95.92 $\pm$ 0.05} \\
&resampling & 86.81 $\pm$ 1.26 &92.72 $\pm$ 0.28 &44.17 $\pm$ 1.15 &95.58 $\pm$ 0.03 \\
&group DRO (EG) &86.11 $\pm$ 1.96 &92.33 $\pm$ 0.65 &42.92 $\pm$ 0.91 &95.82 $\pm$ 0.07 \\
&group DRO (greedy) &88.19 $\pm$ 2.31 &92.65 $\pm$ 0.20 &45.97 $\pm$ 1.73 &95.81 $\pm$ 0.09 \\
\cmidrule{2-6}
&GC-DRO &\textbf{88.75 $\pm$ 0.82} &92.92 $\pm$ 0.16 & \textbf{82.85 $\pm$ 1.54} & 89.32 $\pm$ 2.21 \\
\midrule
\multirow{5}{*}{MNLI} &ERM &70.84 $\pm$ 2.47 &\textbf{86.18 $\pm$ 0.18} &70.84 $\pm$ 2.47 &\textbf{86.18 $\pm$ 0.18} \\
&resampling &67.02 $\pm$ 2.43 &85.72 $\pm$ 0.37 &67.26 $\pm$ 1.63 &85.22 $\pm$ 0.58 \\
&group DRO (EG) &\textbf{77.88 $\pm$ 1.36} &85.16 $\pm$ 0.44 &69.66 $\pm$ 1.98 &84.96 $\pm$ 0.56 \\
&group DRO (greedy)&75.14 $\pm$ 3.96 &85.82 $\pm$ 0.24 &70.34 $\pm$ 2.19 &86.02 $\pm$ 0.25 \\
\cmidrule{2-6}
&GC-DRO &77.82 $\pm$ 1.45 &85.04 $\pm$ 0.67 &\textbf{75.32 $\pm$ 0.93} &84.82 $\pm$ 0.74 \\
\midrule
\multirow{5}{*}{FDCL18} & ERM &34.30 $\pm$ 1.83 &\textbf{79.70 $\pm$ 1.05} &34.30 $\pm$ 1.83 &79.70 $\pm$ 1.05 \\
&resampling &55.44 $\pm$ 4.69 &72.04 $\pm$ 1.99 & 26.10 $\pm$ 4.11 &\textbf{80.66 $\pm$ 0.52} \\
& group DRO (EG) & 55.98 $\pm$ 1.67 &70.06 $\pm$ 3.06 &35.20 $\pm$ 2.24 &79.58 $\pm$ 0.95 \\
&group DRO (greedy) &56.83 $\pm$ 2.94 &70.52 $\pm$ 1.99 &36.24 $\pm$ 3.80 &79.40 $\pm$ 1.12 \\
\cmidrule{2-6}
&GC-DRO &\textbf{57.28 $\pm$ 2.71} &70.26 $\pm$ 0.94 &\textbf{48.42 $\pm$ 6.72} &72.02 $\pm$ 2.96 \\
\bottomrule
\end{tabular}
\vspace{-1mm}
\caption{Robust and average test accuracy and standard deviation on the three tasks.}
\vspace{-2mm}
\label{tab:res}
\end{table*}

\begin{figure*}[t]
    % \vspace{-2mm}
    \centering
    \includegraphics[width=0.95\textwidth]{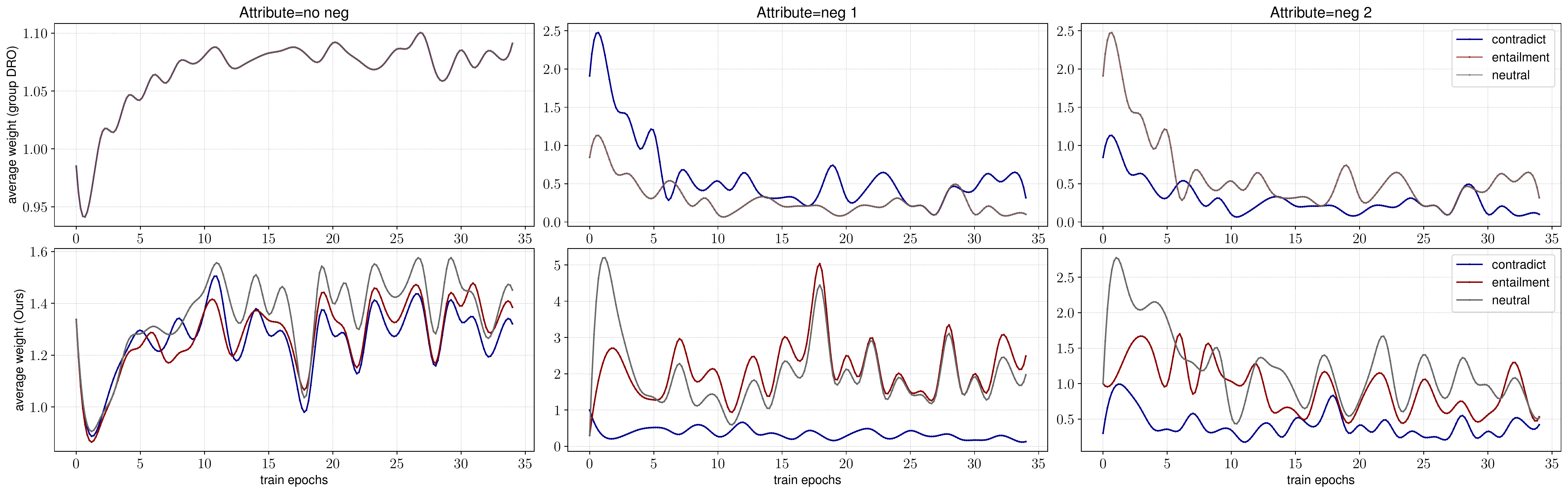}
    \vspace{-4mm}
    \caption{Under the imperfect partition of the MNLI dataset, the aggregated average training weights of instance losses in each group divided by attributes and labels (top: group DRO; bottom: GC-DRO).}
    \label{fig:weights}
    \vspace{-4mm}
\end{figure*}

\vspace{-1mm}
\noindent \textbf{Natural Language Inference (NLI).}
NLI is the task of determining whether a hypothesis is true (entailment), false (contradiction) or undetermined (neutral) given a premise. 
We use the \textbf{MultiNLI} dataset~\citep{williams2018broad} and follow the train/dev/test split in \citet{sagawa2019distributionally}, which results in 206,175 training examples.
\citet{gururangan2018annotation} have shown that there is spurious correlation between the label of contradiction and the presence of negation words (\textit{nobody, nothing, no, never}) due to annotation artifacts. 
We further split the negation words into two groups: set 1 (\textit{nobody, nothing}) and set 2 (\textit{no, never}) to have more variety in the attributes, i.e. $\mathcal{A}=\{$no negation, negation 1, negation 2$\}$, which together with labels forms 9 groups in the clean partition.  
We create \textbf{3} groups in the imperfect partition as shown in Tab.~\ref{tab:db}, where $G_1$ only contains examples from $a=$no negation, while $G_2$ and $G_3$ contain data from both $a=$negation 1 and $a=$negation 2. This causes a dilemma when upweighting either of the groups. 

\vspace{-1mm}
\noindent \textbf{Toxicity Detection.} This task aims to identity various forms of toxic languages (e.g. abusive speech, hate speech), an application with practical and important real-world consequences.
\citet{sap2019risk} have shown that there is a strong correlation between certain surface markers of English spoken by minority groups and the labels of toxicity. And such biases can be acquired and propagated by models trained on these corpora.
We perform experiments on the \textbf{FDCL18}~\citep{fortuna2018survey} dataset, a corpus of 100k tweets annotated with four labels: $\mathcal{Y} = \{$hateful, spam, abusive and normal$\}$. 
Since the dataset does not contain the dialect information, we follow \citet{sap2019risk} and use annotations predicted by the dialect classifier in~\citet{blodgett2016demographic} to label each example as one of four English varieties: $\mathcal{A}=\{$White-aligned, African American (AAE), Hispanic, and Other$\}$.
As noted in \citet{sap2019risk}, these automatically obtained labels correlate highly with self-reported race and provide an accurate proxy for the dialect labels.
These dialect attributes and toxicity labels together divide the dataset into 16 groups in the clean partition. 
To construct the imperfect partitions, we investigate a natural setting where data is divided by the dialect attributes, therefore we have \textbf{4} groups in the imperfect partition. 
The test set of FDCL18 contains some groups that are severely under-represented. 
In order to make the robust accuracy reliable yet still representative of the under-represented groups, we follow~\citet{michel2021modeling} and combine groups that contain less than 100 samples into a single group to compute robust test accuracies.

\subsection{Experimental Setup}
We fine-tune pretrained models for object recognition and NLP tasks that achieve high average test accuracies, specifically ResNet18~\citep{he2016deep} on CelebA and RoBERTa~\citep{liu2019roberta} on the MultiNLI and FDCL18 datasets. 
We select hyperparameters by the robust validation accuracy. For the clean partitions, we set $\alpha=0.2, \beta=0.5$ for all the three tasks. For the imperfect partitions, we set a relatively lower value of $\beta$ to highlight badly performed instances within groups.
Specifically, for NLP tasks we set $\alpha=0.5$, $\beta=0.2$ and $0.25$ for NLI and toxicity detection respectively, and for the image task, we set $\alpha=0.2, \beta=0.1$.
For more training details, see Appendix~\ref{app:exp}.
We measure both the \emph{average} accuracy over all the test data as well as the \emph{robust} accuracy (worst accuracy across all groups).
Even though different partitions (clean/imperfect) are used at training time, we always evaluate the model's robust accuracy across groups of the clean partitions of the test data.

We compare with \textbf{ERM}, which minimizes the average training loss on the empirical training distribution, formally
\vspace{-1mm}
\begin{equation}
\small
    \mathcal{L}_{\mathrm{ERM}}(\theta) = \mathbb{E}_{(x, y)\sim P_{train}}[\ell(x, y;\theta)]
\vspace{-1mm}
\end{equation}
We also compare with two variants of group DRO with different objective and optimization procedures: a \textbf{greedy}~\citep{oren2019distributionally} algorithm for CVaR-group DRO and a \textbf{exponentiated-gradient (EG)}~\citep{sagawa2019distributionally} procedure with full simplex .
Note that while previous work~\citep{sagawa2019distributionally} has found the greedy algorithm is unstable and underperforms EG, 
% we found an issue in \citet{oren2019distributionally}'s implementation of the greedy optimization algorithm where loss is averaged over groups instead of the mini-batch.
% which leads to downweighting losses of examples in the large groups (for details, see Appendix~\ref{app:bug}).
we did not observe this issue with our implementation where we took a slight different approach to compute the worst expected loss and we detail this difference in Appendix~\ref{app:bug}.
% We use a version that remedies this issue.
In addition, we compare with the \textbf{resampling} method, which optimizes on minibatches sampled from uniform group frequencies, which is often used for imbalanced datasets.

\subsection{Main Results}
We present the robust and average test accuracies of all three tasks under different partitions in Tab.~\ref{tab:res}.
Models are selected based on the worst-performing accuracy of group (of the clean partition) in the validation set. All the results are averaged over 5 runs with different random seeds. 
Except for ERM, all the models leverage the group information at training time. 
\textbf{First}, as expected, ERM models attain high average test accuracies across all the datasets but perform poorly on the worst-case group. 
\textbf{Second}, we observe that under the clean partition, group DRO models always significantly outperform ERM on the worst-group test accuracy with modest drop in the average test accuracy. And we also note that group DRO optimized with the greedy algorithm performs on par with that optimized by the EG based algorithm.
By contrast, the resampling method can not consistently perform well on the worst test groups on all datasets. 
Furthermore, our method performs similarly to or slightly better than group DRO on all three datasets under the clean partition.
\textbf{Third}, under the imperfect partition, neither group DRO nor resampling can perform well in the worst test groups and achieves similar performance to that of ERM models. 
On the other hand, our method performs significantly better in terms of the robust accuracy on all three datasets, with 5-37 points in improvement over group DRO models. Although the results are worse compared to those under the clean partition, we demonstrate that our method is much more agnostic to the underlying data partitions.

\vspace{-1mm}
\subsection{Analysis}
\label{sec:analysis}
\noindent \textbf{Why does GC-DRO perform well on robust accuracy?}
In this section, we investigate why group-conditional DRO works well under imperfect partitions. To do this, we first compute the actual weight ($\frac{n_i q^{(t)}(g_i)q^{(t)}(x, y|g_i)}{\hat{p}^{train(t)}(g_i)}$ in Alg.~\ref{alg:sgdro}) applied to each instance $(x_i, y_i)$ at every step $t$ for group DRO and our method respectively. 
The groups in imperfect partitions contain instances from different values of $(a, y)$ and to study the effects of learning weights on the test groups, we aggregate the weights of instances on each group $(a, y)$ (i.e. the clean partition) by taking average over all steps in each epoch.
In Fig.~\ref{fig:weights}, we plot the dynamic aggregated weights over the training course learned with the imperfect partition (3 groups, see Tab.~\ref{tab:db}) of the MNLI dataset.
We observe that GC-DRO can assign higher weights to instances that belong to groups of $(a= $negation, $y\neq$ contradiction$)$, which helps prevent the model from learning the spurious correlations between $a=$ negation and $y=$contradiction.
On the contrary, group DRO can not accomplish this goal because it can not handle these subgroups inside groups in a fine-grained way.
\begin{figure}[h]
    \vspace{-2mm}
    \centering
    \includegraphics[width=0.45\textwidth]{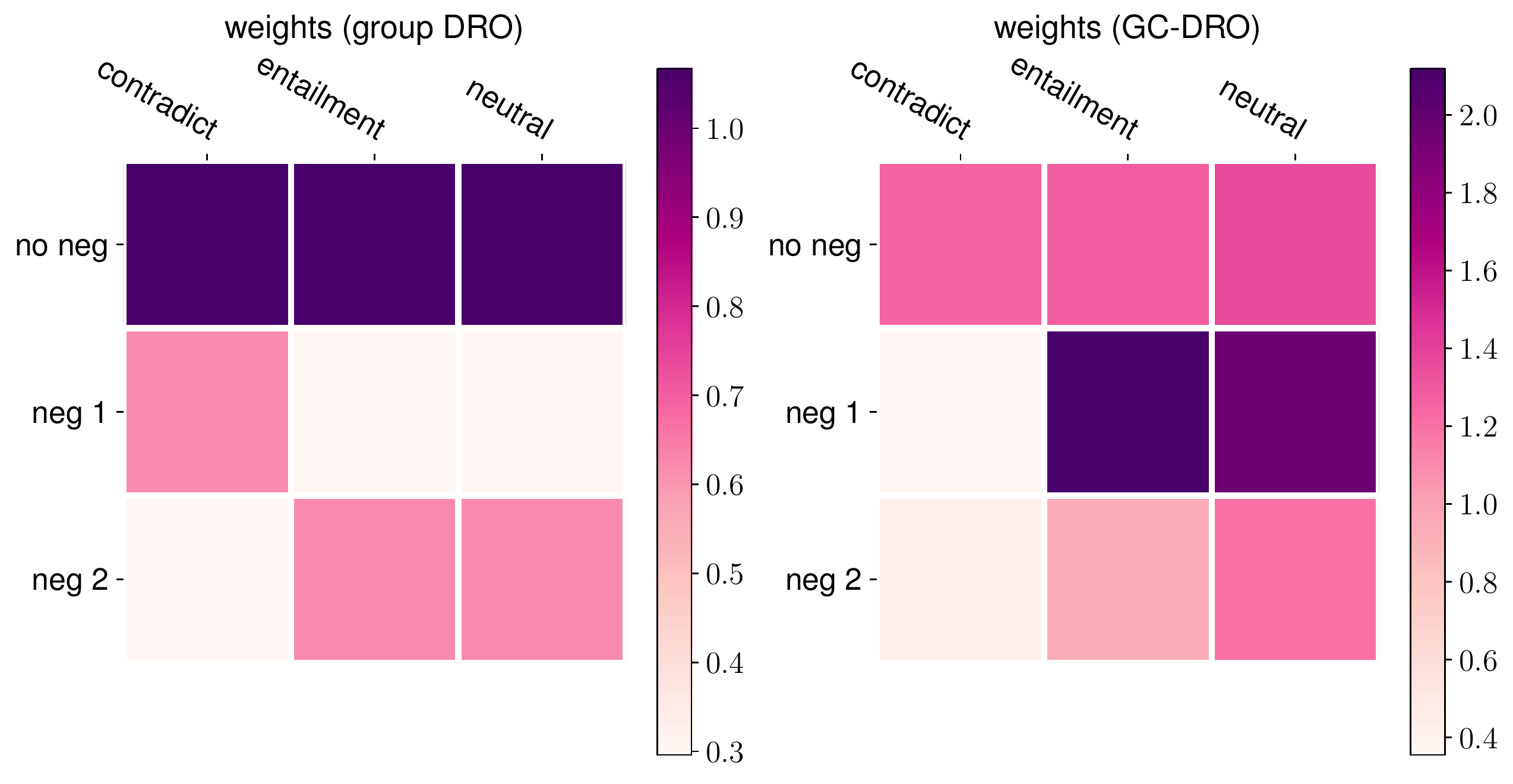}
    \vspace{-6mm}
    \caption{The heatmap of summarized learning weights for different groups from group DRO and our method.}
    \label{fig:heatmap}
    \vspace{-3mm}
\end{figure}

To make this trend more clear, we summarize the weights across all the epochs for each group of $(a, y)$ and present the heat map in Fig.~\ref{fig:heatmap}. 
We can see that group DRO focuses on learning from the large group that does not contain negation words but pays less attention to those minority groups. On the contrary, our method encourages the model to learn from minority groups that can help combat spurious features.

\begin{table}[t]
\centering
\small
\resizebox{0.4\textwidth}{!}{
\begin{tabular}{lcc}
\toprule
&Robust Acc &Average Acc \\
\midrule
ERM & 34.30 $\pm$ 1.83 &79.7 $\pm$ 1.05 \\
resampling &34.20 $\pm$ 2.36 &79.4 $\pm$ 1.24 \\
group DRO (EG) &32.84 $\pm$ 2.72 &\textbf{80.5 $\pm$ 0.59} \\
group DRO (greedy) &34.48 $\pm$ 4.69 &79.62 $\pm$ 0.59 \\
\midrule
GC-DRO (ours) &\textbf{45.06 $\pm$ 6.77} &70.7 $\pm$ 4.81 \\
\bottomrule
\end{tabular}}
\caption{Average and robust test accuracies of FDCL18 under the partitions via unsupervised clustering.}
\label{tab:bert:group}
\vspace{-3mm}
\end{table}

\noindent \textbf{On groups produced by unsupervised clustering.}
We study a more realistic setting where no group information is available and we use an unsupervised clustering algorithm to produce the partitions. 
Specifically, we first embed the training sentences of the FDCL18 dataset with Sentence-BERT~\citep{reimers2019sentence}, a well-performing semantic sentence embedder, then we use K-means to obtain 8 clusters. 
In Tab.~\ref{tab:bert:group}, we show the robust and average accuracy on the test set of the toxicity detection task. 
Our method once again significantly outperforms other baseline methods on the robust test accuracy, which demonstrates the robustness of GC-DRO under different partitions.

\begin{figure}[h]
    \vspace{-2mm}
    \centering
    \includegraphics[width=0.5\textwidth]{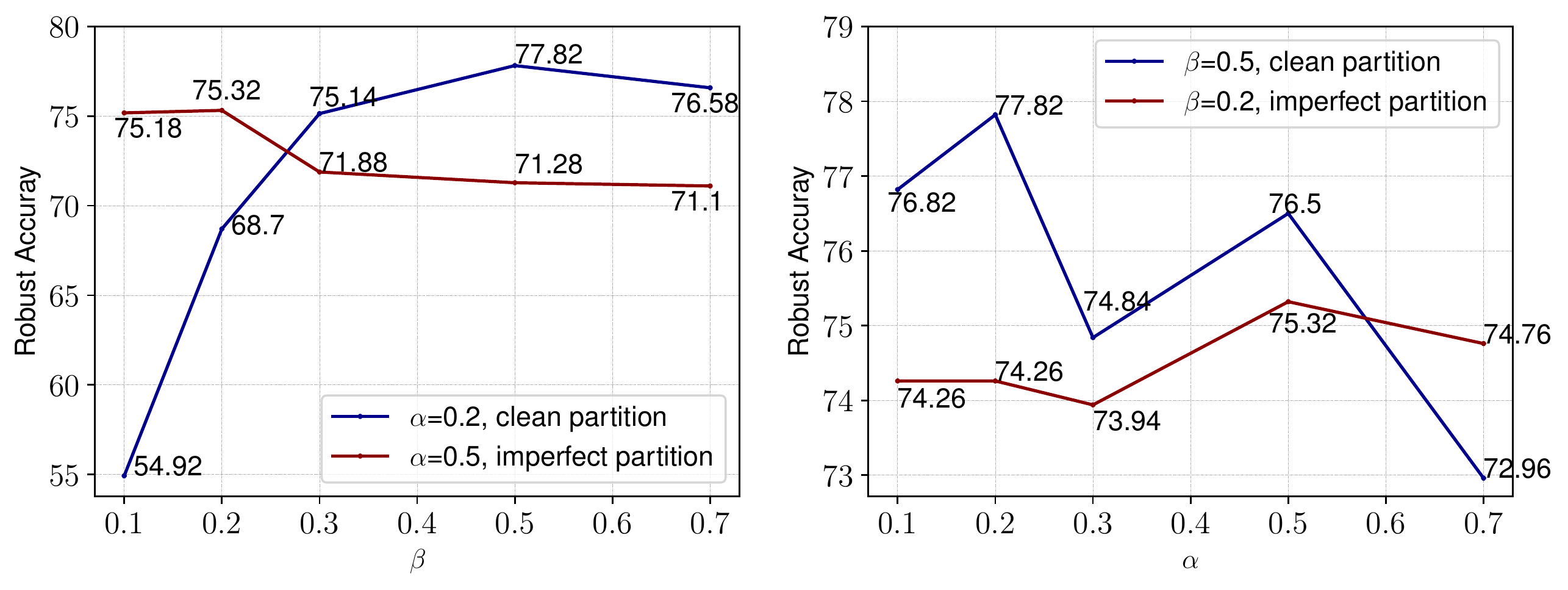}
    \vspace{-8mm}
    \caption{Ablation studies on $\alpha$ and $\beta$ on the MNLI datasets.}
    \label{fig:ablation}
    \vspace{-3mm}
\end{figure}
\noindent \textbf{Ablation studies on $\alpha$ and $\beta$.}
We perform ablation studies on the two important hyperparameters $\alpha$ and $\beta$ used in our method. 
In Fig~\ref{fig:ablation}, we fix one value and vary the other and plot the robust test accuracies over 5 random runs (the variance of average test accuracies is very small) on the NLI task. 
We observe that GC-DRO is less sensitive to different combinations of $\alpha$ and $\beta$ under the imperfect partitions.
However, for the clean partitions, a larger $\beta$ and a smaller $\alpha$ tends to yield better performance, as GC-DRO behaves more close to the plain group DRO.
\section{Conclusion}
Through a mathematical characterization of features used in prediction, we have demonstrated that under covariate shift ERM models can pick up spurious features or miss robust features.
The GC-DRO algorithm resulting from this analysis allows for a more flexible uncertainty set that performs consistently well in the worst test group under different partitions.
This new understanding of features opens up new avenues in both redesigning our distributionally robust algorithms, and further characterizing possible spurious factors that may influence model robustness, for example through unsupervised learning.

\section*{Acknowledgements}

This work in this paper was supported in part by a Facebook SRA Award and the NSF/Amazon Fairness in AI program under grant number 2040926. 
We thank the anonymous reviewers for useful suggestions.

% Acknowledgements should only appear in the accepted version.
% \section*{Acknowledgements}

% In the unusual situation where you want a paper to appear in the
% references without citing it in the main text, use \nocite
\nocite{langley00}

% \clearpage
\bibliography{main}

\begin{thebibliography}{54}
\providecommand{\natexlab}[1]{#1}
\providecommand{\url}[1]{\texttt{#1}}
\expandafter\ifx\csname urlstyle\endcsname\relax
  \providecommand{\doi}[1]{doi: #1}\else
  \providecommand{\doi}{doi: \begingroup \urlstyle{rm}\Url}\fi

\bibitem[Achille \& Soatto(2018)Achille and Soatto]{achille2018emergence}
Achille, A. and Soatto, S.
\newblock Emergence of invariance and disentanglement in deep representations.
\newblock \emph{The Journal of Machine Learning Research}, 19\penalty0
  (1):\penalty0 1947--1980, 2018.

\bibitem[Adragni \& Cook(2009)Adragni and Cook]{adragni2009sufficient}
Adragni, K.~P. and Cook, R.~D.
\newblock Sufficient dimension reduction and prediction in regression.
\newblock \emph{Philosophical Transactions of the Royal Society A:
  Mathematical, Physical and Engineering Sciences}, 367\penalty0
  (1906):\penalty0 4385--4405, 2009.

\bibitem[Arjovsky et~al.(2019)Arjovsky, Bottou, Gulrajani, and
  Lopez-Paz]{arjovsky2019invariant}
Arjovsky, M., Bottou, L., Gulrajani, I., and Lopez-Paz, D.
\newblock Invariant risk minimization.
\newblock \emph{arXiv preprint arXiv:1907.02893}, 2019.

\bibitem[Bahng et~al.(2020)Bahng, Chun, Yun, Choo, and Oh]{bahng2020learning}
Bahng, H., Chun, S., Yun, S., Choo, J., and Oh, S.~J.
\newblock Learning de-biased representations with biased representations.
\newblock In \emph{International Conference on Machine Learning}, pp.\
  528--539. PMLR, 2020.

\bibitem[Beery et~al.(2018)Beery, Van~Horn, and Perona]{beery2018recognition}
Beery, S., Van~Horn, G., and Perona, P.
\newblock Recognition in terra incognita.
\newblock In \emph{Proceedings of the European Conference on Computer Vision
  (ECCV)}, pp.\  456--473, 2018.

\bibitem[Bellot \& van~der Schaar(2020)Bellot and van~der
  Schaar]{bellot2020generalization}
Bellot, A. and van~der Schaar, M.
\newblock Generalization and invariances in the presence of unobserved
  confounding.
\newblock \emph{arXiv preprint arXiv:2007.10653}, 2020.

\bibitem[Ben-Tal et~al.(2013)Ben-Tal, Den~Hertog, De~Waegenaere, Melenberg, and
  Rennen]{ben2013robust}
Ben-Tal, A., Den~Hertog, D., De~Waegenaere, A., Melenberg, B., and Rennen, G.
\newblock Robust solutions of optimization problems affected by uncertain
  probabilities.
\newblock \emph{Management Science}, 59\penalty0 (2):\penalty0 341--357, 2013.

\bibitem[Blodgett et~al.(2016)Blodgett, Green, and
  O’Connor]{blodgett2016demographic}
Blodgett, S.~L., Green, L., and O’Connor, B.
\newblock Demographic dialectal variation in social media: A case study of
  african-american english.
\newblock In \emph{Proceedings of the 2016 Conference on Empirical Methods in
  Natural Language Processing}, pp.\  1119--1130, 2016.

\bibitem[Buolamwini \& Gebru(2018)Buolamwini and Gebru]{buolamwini2018gender}
Buolamwini, J. and Gebru, T.
\newblock Gender shades: Intersectional accuracy disparities in commercial
  gender classification.
\newblock In \emph{Conference on fairness, accountability and transparency},
  pp.\  77--91, 2018.

\bibitem[Cover(1999)]{cover1999elements}
Cover, T.~M.
\newblock \emph{Elements of information theory}.
\newblock John Wiley \& Sons, 1999.

\bibitem[Cvitkovic \& Koliander(2019)Cvitkovic and
  Koliander]{cvitkovic2019minimal}
Cvitkovic, M. and Koliander, G.
\newblock Minimal achievable sufficient statistic learning.
\newblock volume~97, pp.\  1465--1474. PMLR, 2019.

\bibitem[David et~al.(2010)David, Lu, Luu, and P{\'a}l]{david2010impossibility}
David, S.~B., Lu, T., Luu, T., and P{\'a}l, D.
\newblock Impossibility theorems for domain adaptation.
\newblock In \emph{Proceedings of the Thirteenth International Conference on
  Artificial Intelligence and Statistics}, pp.\  129--136. JMLR Workshop and
  Conference Proceedings, 2010.

\bibitem[Duchi et~al.(2016)Duchi, Glynn, and Namkoong]{duchi2016statistics}
Duchi, J., Glynn, P., and Namkoong, H.
\newblock Statistics of robust optimization: A generalized empirical likelihood
  approach.
\newblock \emph{arXiv preprint arXiv:1610.03425}, 2016.

\bibitem[Duchi et~al.(2019)Duchi, Hashimoto, and
  Namkoong]{duchi2019distributionally}
Duchi, J.~C., Hashimoto, T., and Namkoong, H.
\newblock Distributionally robust losses against mixture covariate shifts.
\newblock 2019.

\bibitem[Dynkin(2000)]{dynkin2000necessary}
Dynkin, E.~B.
\newblock Necessary and sufficient statistics for a family of probability
  distributions.
\newblock \emph{Selected Papers of EB Dynkin with Commentary}, 14:\penalty0
  393, 2000.

\bibitem[Fortuna \& Nunes(2018)Fortuna and Nunes]{fortuna2018survey}
Fortuna, P. and Nunes, S.
\newblock A survey on automatic detection of hate speech in text.
\newblock \emph{ACM Computing Surveys (CSUR)}, 51\penalty0 (4):\penalty0 1--30,
  2018.

\bibitem[Gao \& Kleywegt(2016)Gao and Kleywegt]{gao2016distributionally}
Gao, R. and Kleywegt, A.~J.
\newblock Distributionally robust stochastic optimization with wasserstein
  distance.
\newblock \emph{arXiv preprint arXiv:1604.02199}, 2016.

\bibitem[Geiger(2020)]{geiger2020information}
Geiger, B.~C.
\newblock On information plane analyses of neural network classifiers--a
  review.
\newblock \emph{arXiv preprint arXiv:2003.09671}, 2020.

\bibitem[Goyal et~al.(2017)Goyal, Khot, Summers-Stay, Batra, and
  Parikh]{goyal2017making}
Goyal, Y., Khot, T., Summers-Stay, D., Batra, D., and Parikh, D.
\newblock Making the v in vqa matter: Elevating the role of image understanding
  in visual question answering.
\newblock In \emph{Proceedings of the IEEE Conference on Computer Vision and
  Pattern Recognition}, pp.\  6904--6913, 2017.

\bibitem[Gururangan et~al.(2018)Gururangan, Swayamdipta, Levy, Schwartz,
  Bowman, and Smith]{gururangan2018annotation}
Gururangan, S., Swayamdipta, S., Levy, O., Schwartz, R., Bowman, S., and Smith,
  N.~A.
\newblock Annotation artifacts in natural language inference data.
\newblock In \emph{Proceedings of the 2018 Conference of the North American
  Chapter of the Association for Computational Linguistics: Human Language
  Technologies, Volume 2 (Short Papers)}, pp.\  107--112, 2018.

\bibitem[Hashimoto et~al.(2018)Hashimoto, Srivastava, Namkoong, and
  Liang]{hashimoto2018fairness}
Hashimoto, T., Srivastava, M., Namkoong, H., and Liang, P.
\newblock Fairness without demographics in repeated loss minimization.
\newblock In \emph{International Conference on Machine Learning}, pp.\
  1929--1938, 2018.

\bibitem[He et~al.(2016)He, Zhang, Ren, and Sun]{he2016deep}
He, K., Zhang, X., Ren, S., and Sun, J.
\newblock Deep residual learning for image recognition.
\newblock In \emph{Proceedings of the IEEE conference on computer vision and
  pattern recognition}, pp.\  770--778, 2016.

\bibitem[Hochreiter \& Schmidhuber(1997)Hochreiter and
  Schmidhuber]{hochreiter1997long}
Hochreiter, S. and Schmidhuber, J.
\newblock Long short-term memory.
\newblock \emph{Neural computation}, 9\penalty0 (8):\penalty0 1735--1780, 1997.

\bibitem[Hu et~al.(2018)Hu, Niu, Sato, and Sugiyama]{hu2018does}
Hu, W., Niu, G., Sato, I., and Sugiyama, M.
\newblock Does distributionally robust supervised learning give robust
  classifiers?
\newblock In \emph{International Conference on Machine Learning}, pp.\
  2029--2037. PMLR, 2018.

\bibitem[Hu \& Hong(2013)Hu and Hong]{hu2013kullback}
Hu, Z. and Hong, L.~J.
\newblock Kullback-leibler divergence constrained distributionally robust
  optimization.
\newblock \emph{Available at Optimization Online}, 2013.

\bibitem[Husain(2020)]{husain2020distributional}
Husain, H.
\newblock Distributional robustness with ipms and links to regularization and
  gans.
\newblock In \emph{Proceedings of the 34th Annual Conference on Neural
  Information Processing Systems (NIPS)}, 2020.

\bibitem[Kingma \& Ba(2014)Kingma and Ba]{kingma2014adam}
Kingma, D.~P. and Ba, J.
\newblock Adam: A method for stochastic optimization.
\newblock \emph{arXiv preprint arXiv:1412.6980}, 2014.

\bibitem[Koenecke et~al.(2020)Koenecke, Nam, Lake, Nudell, Quartey, Mengesha,
  Toups, Rickford, Jurafsky, and Goel]{koenecke2020racial}
Koenecke, A., Nam, A., Lake, E., Nudell, J., Quartey, M., Mengesha, Z., Toups,
  C., Rickford, J.~R., Jurafsky, D., and Goel, S.
\newblock Racial disparities in automated speech recognition.
\newblock \emph{Proceedings of the National Academy of Sciences}, 117\penalty0
  (14):\penalty0 7684--7689, 2020.

\bibitem[Koh et~al.(2020)Koh, Sagawa, Marklund, Xie, Zhang, Balsubramani, Hu,
  Yasunaga, Phillips, Beery, et~al.]{koh2020wilds}
Koh, P.~W., Sagawa, S., Marklund, H., Xie, S.~M., Zhang, M., Balsubramani, A.,
  Hu, W., Yasunaga, M., Phillips, R.~L., Beery, S., et~al.
\newblock Wilds: A benchmark of in-the-wild distribution shifts.
\newblock \emph{arXiv preprint arXiv:2012.07421}, 2020.

\bibitem[Kolchinsky et~al.(2019)Kolchinsky, Tracey, and
  Van~Kuyk]{kolchinsky2018caveats}
Kolchinsky, A., Tracey, B.~D., and Van~Kuyk, S.
\newblock Caveats for information bottleneck in deterministic scenarios.
\newblock In \emph{International Conference on Learning Representations}, 2019.

\bibitem[Kullback \& Leibler(1951)Kullback and
  Leibler]{kullback1951information}
Kullback, S. and Leibler, R.~A.
\newblock On information and sufficiency.
\newblock \emph{The annals of mathematical statistics}, 22\penalty0
  (1):\penalty0 79--86, 1951.

\bibitem[Lewis(2013)]{lewis2013counterfactuals}
Lewis, D.
\newblock \emph{Counterfactuals}.
\newblock John Wiley \& Sons, 2013.

\bibitem[Liu et~al.(2019)Liu, Ott, Goyal, Du, Joshi, Chen, Levy, Lewis,
  Zettlemoyer, and Stoyanov]{liu2019roberta}
Liu, Y., Ott, M., Goyal, N., Du, J., Joshi, M., Chen, D., Levy, O., Lewis, M.,
  Zettlemoyer, L., and Stoyanov, V.
\newblock Roberta: A robustly optimized bert pretraining approach.
\newblock \emph{arXiv preprint arXiv:1907.11692}, 2019.

\bibitem[Liu et~al.(2015)Liu, Luo, Wang, and Tang]{liu2015deep}
Liu, Z., Luo, P., Wang, X., and Tang, X.
\newblock Deep learning face attributes in the wild.
\newblock In \emph{Proceedings of the IEEE international conference on computer
  vision}, pp.\  3730--3738, 2015.

\bibitem[Lopez-Paz(2016)]{lopez2016dependence}
Lopez-Paz, D.
\newblock \emph{From dependence to causation}.
\newblock PhD thesis, University of Cambridge, 2016.

\bibitem[Lovering et~al.(2021)Lovering, Jha, Linzen, and
  Pavlick]{lovering2020info}
Lovering, C., Jha, R., Linzen, T., and Pavlick, E.
\newblock Predicting inductive biases of pre-trained models.
\newblock In \emph{International Conference on Learning Representations
  (ICLR)}, 2021.

\bibitem[McCoy et~al.(2019)McCoy, Pavlick, and Linzen]{mccoy2019right}
McCoy, T., Pavlick, E., and Linzen, T.
\newblock Right for the wrong reasons: Diagnosing syntactic heuristics in
  natural language inference.
\newblock In \emph{Proceedings of the 57th Annual Meeting of the Association
  for Computational Linguistics}, pp.\  3428--3448, 2019.

\bibitem[Michel et~al.(2021)Michel, Hashimoto, and Neubig]{michel2021modeling}
Michel, P., Hashimoto, T., and Neubig, G.
\newblock Modeling the second player in distributionally robust optimization.
\newblock In \emph{International Conference on Learning Representations
  (ICLR)}, 2021.

\bibitem[Oren et~al.(2019)Oren, Sagawa, Hashimoto, and
  Liang]{oren2019distributionally}
Oren, Y., Sagawa, S., Hashimoto, T.~B., and Liang, P.
\newblock Distributionally robust language modeling.
\newblock In \emph{Conference on Empirical Methods in Natural Language
  Processing (EMNLP)}, Hong Kong, November 2019.

\bibitem[Ott et~al.(2019)Ott, Edunov, Baevski, Fan, Gross, Ng, Grangier, and
  Auli]{ott2019fairseq}
Ott, M., Edunov, S., Baevski, A., Fan, A., Gross, S., Ng, N., Grangier, D., and
  Auli, M.
\newblock fairseq: A fast, extensible toolkit for sequence modeling.
\newblock In \emph{Proceedings of NAACL-HLT 2019: Demonstrations}, 2019.

\bibitem[Reimers et~al.(2019)Reimers, Gurevych, Reimers, Gurevych, Thakur,
  Reimers, Daxenberger, and Gurevych]{reimers2019sentence}
Reimers, N., Gurevych, I., Reimers, N., Gurevych, I., Thakur, N., Reimers, N.,
  Daxenberger, J., and Gurevych, I.
\newblock Sentence-bert: Sentence embeddings using siamese bert-networks.
\newblock In \emph{Proceedings of the 2019 Conference on Empirical Methods in
  Natural Language Processing}. Association for Computational Linguistics,
  2019.

\bibitem[Rockafellar et~al.(2000)Rockafellar, Uryasev,
  et~al.]{rockafellar2000optimization}
Rockafellar, R.~T., Uryasev, S., et~al.
\newblock Optimization of conditional value-at-risk.
\newblock \emph{Journal of risk}, 2:\penalty0 21--42, 2000.

\bibitem[Sagawa et~al.(2020{\natexlab{a}})Sagawa, Koh, Hashimoto, and
  Liang]{sagawa2019distributionally}
Sagawa, S., Koh, P.~W., Hashimoto, T.~B., and Liang, P.
\newblock Distributionally robust neural networks for group shifts: On the
  importance of regularization for worst-case generalization.
\newblock In \emph{International Conference on Learning Representations
  (ICLR)}, Addis Ababa, Ethiopia, April 2020{\natexlab{a}}.

\bibitem[Sagawa et~al.(2020{\natexlab{b}})Sagawa, Raghunathan, Koh, and
  Liang]{sagawa2020investigation}
Sagawa, S., Raghunathan, A., Koh, P.~W., and Liang, P.
\newblock An investigation of why overparameterization exacerbates spurious
  correlations.
\newblock In \emph{International Conference on Machine Learning (ICML)}, July
  2020{\natexlab{b}}.

\bibitem[Sap et~al.(2019)Sap, Card, Gabriel, Choi, and Smith]{sap2019risk}
Sap, M., Card, D., Gabriel, S., Choi, Y., and Smith, N.~A.
\newblock The risk of racial bias in hate speech detection.
\newblock In \emph{Proceedings of the 57th annual meeting of the association
  for computational linguistics}, pp.\  1668--1678, 2019.

\bibitem[Schwartz-Ziv \& Tishby(2017)Schwartz-Ziv and
  Tishby]{shwartz2017opening}
Schwartz-Ziv, R. and Tishby, N.
\newblock Opening the black box of deep neural networks via information.
\newblock \emph{arXiv preprint arXiv:1703.00810}, 2017.

\bibitem[Shamir et~al.(2010)Shamir, Sabato, and Tishby]{shamir2010learning}
Shamir, O., Sabato, S., and Tishby, N.
\newblock Learning and generalization with the information bottleneck.
\newblock volume 411, pp.\  2696--2711. Elsevier, 2010.

\bibitem[Sohoni et~al.(2020)Sohoni, Dunnmon, Angus, Gu, and
  R{\'e}]{sohoni2020no}
Sohoni, N., Dunnmon, J., Angus, G., Gu, A., and R{\'e}, C.
\newblock No subclass left behind: Fine-grained robustness in coarse-grained
  classification problems.
\newblock \emph{34th Conference on Neural Information Processing Systems
  (NeurIPS)}, 33, 2020.

\bibitem[Tenenbaum(2018)]{tenenbaum2018building}
Tenenbaum, J.
\newblock Building machines that learn and think like people.
\newblock In \emph{Proceedings of the 17th International Conference on
  Autonomous Agents and MultiAgent Systems}, pp.\  5--5, 2018.

\bibitem[Tishby et~al.(2000)Tishby, Pereira, and Bialek]{tishby2000information}
Tishby, N., Pereira, F.~C., and Bialek, W.
\newblock The information bottleneck method.
\newblock \emph{arXiv preprint physics/0004057}, 2000.

\bibitem[Torralba \& Efros(2011)Torralba and Efros]{torralba2011unbiased}
Torralba, A. and Efros, A.~A.
\newblock Unbiased look at dataset bias.
\newblock In \emph{CVPR 2011}, pp.\  1521--1528. IEEE, 2011.

\bibitem[Wang et~al.(2020)Wang, Guo, Narasimhan, Cotter, Gupta, and
  Jordan]{wang2020robust}
Wang, S., Guo, W., Narasimhan, H., Cotter, A., Gupta, M.~R., and Jordan, M.~I.
\newblock Robust optimization for fairness with noisy protected groups.
\newblock \emph{34th Conference on Neural Information Processing Systems
  (NeurIPS).}, 2020.

\bibitem[Wen et~al.(2014)Wen, Yu, and Greiner]{wen2014robust}
Wen, J., Yu, C.-N., and Greiner, R.
\newblock Robust learning under uncertain test distributions: Relating
  covariate shift to model misspecification.
\newblock In \emph{ICML}, pp.\  631--639, 2014.

\bibitem[Williams et~al.(2018)Williams, Nangia, and Bowman]{williams2018broad}
Williams, A., Nangia, N., and Bowman, S.
\newblock A broad-coverage challenge corpus for sentence understanding through
  inference.
\newblock In \emph{Proceedings of the 2018 Conference of the North American
  Chapter of the Association for Computational Linguistics: Human Language
  Technologies, Volume 1 (Long Papers)}, pp.\  1112--1122, 2018.

\end{thebibliography}
\bibliographystyle{icml2021}

%%%%%%%%%%%%%%%%%%%%%%%%%%%%%%%%%%%%%%%%%%%%%%%%%%%%%%%%%%%%%%%%%%%%%%%%%%%%%%%
%%%%%%%%%%%%%%%%%%%%%%%%%%%%%%%%%%%%%%%%%%%%%%%%%%%%%%%%%%%%%%%%%%%%%%%%%%%%%%%
% DELETE THIS PART. DO NOT PLACE CONTENT AFTER THE REFERENCES!
%%%%%%%%%%%%%%%%%%%%%%%%%%%%%%%%%%%%%%%%%%%%%%%%%%%%%%%%%%%%%%%%%%%%%%%%%%%%%%%
\clearpage
\appendix
\onecolumn
\section{Proofs of Theorem 1}
\label{append:covariate:proof}
\begin{lemma}\label{lem:1}
If $T$ is sufficient statistics, we have $p(Y,X|T)=p(Y|T)\cdot p(X|T)$.
\end{lemma}
\begin{lemma}\label{lem:2}
If $T$ is sufficient statistics, we have $p(Y|T(X)) = p(Y|X)$.
\end{lemma}
\begin{proof}
 Find $T'(X)$, s.t. $S(x) = \langle T(X), T'(X) \rangle$ is an invertible mapping of $X$, thus $p(Y|X)=p(Y|S(X))=p(Y|T(X), T'(X))$. We have,
 \begin{equation}
     p(Y, T(X), T'(X)|T(X)) = p(Y|T'(X), T(X))p(T'(X)|T(X))
     \label{eq:1}
 \end{equation}
 From Lemma~\ref{lem:1}, we have
 \begin{equation}
     p(Y, T(X), T'(X)|T(X)) = p(Y|T(X))p(T'(X)|T(X))
     \label{eq:2}
 \end{equation}
 By (\ref{eq:1}) and (\ref{eq:2}), we obtain $p(Y|T'(X), T(X)) = p(Y|T(X)) = p(Y|X)$.
\end{proof}

\cov*
\begin{proof}
Since there is covariate shift between $\ptrue$ and $\ptrain$, we have $\ptrain(Y|X) = \ptrue(Y|X), \forall x\in \mathcal{X}_{train}$.
Since $T_{train}(X)$ is MSS of $\ptrain$ and by Lemma~\ref{lem:2}, we have $\ptrain(Y|T_{train}(X)) = \ptrain(Y|X) = \ptrue(Y|X) = \ptrue(Y|T_{ideal}(X)), \forall x\in \mathcal{X}_{train}$.
Then $\forall x\in \mathcal{X}_{train}, y \in \mathcal{Y}$, 
\begin{align}\label{eq:3}
\ptrain(y|T_{ideal}(x)) &= \sum\limits_{x': T_{ideal}(x') = T_{ideal}(x)} \ptrain(y|x') \ptrain(x'|T(x)) \nonumber \\
 &= \sum\limits_{x': T_{ideal}(x') = T_{ideal}(x)} \ptrue(y|x') \ptrain(x'|T(x)) \nonumber \\
 &= \sum\limits_{x': T_{ideal}(x') = T_{ideal}(x)} \ptrue(y|T(x)) \ptrain(x'|T(x)) \nonumber \\
 &= \ptrue(y|T_{ideal}(x))
\end{align}
Then we have
\begin{align}\label{eq:4}
    H_{train}(Y|T_{train}(X)) &= \sum_{x, y} \ptrain(x, y)[-\log \ptrain(y|T_{train}(x))] \nonumber \\
    &= \sum_{x,y} \ptrain(x,y)[-\log \ptrue(y|T_{ideal}(x))] \nonumber \\
    &= \sum_{x,y} \ptrain(x,y)[-\log \ptrain(y|T_{ideal}(x))] \nonumber \\
    &= H_{train}(Y|T_{ideal}(X))
\end{align}
From \eqref{eq:4} and the definition of sufficient statistics, we have
\begin{equation}
     I_{train}(Y; T_{train}(X)) = I_{train}(Y; X) = I_{train}(Y; T_{ideal}(X))
\end{equation}
Thus, $T_{ideal}(X)$ is the sufficient statistics of $X$ about $Y$ under $\ptrain$.
By definition, we have 
\begin{equation}
    H_{train}(T_{train}(X)|T_{ideal}(X)) = 0.
\end{equation}
\end{proof}

\cor*
\begin{proof}
Since $\mathcal{X}_{train} = \mathcal{X}_{ideal} = \mathcal{X}$, with the similar derivation of \eqref{eq:3}, we have $\forall x\in \mathcal{X}, y \in \mathcal{Y}$
\begin{equation}
    \ptrue(y|T_{ideal}(x)) = \ptrue(y|T_{train}(x))
\end{equation}
Together with Theorem~\ref{thm:cov}, we have $T_{train}(x)$ is also the MSS under $\ptrue$.
\end{proof}

\section{Connections between MLE and Learning Minimal Sufficient Statistics}
\label{app:connect}
\subsection{Information Bottleneck (IB) Method}
\label{app:ib}
The information bottleneck (IB) method~\citep{tishby2000information} is an information theoretic principle introduced to extract relevant information that an input $X \in \mathcal{X}$ contains about an output random variable $Y \in \mathcal{Y}$.
Defined on a joint distribution of $X$ and $Y$, IB learns a mapping function $T(X)$ by optimizing the trade-off between the mutual information $I(X; T)$ and $I(Y; T)$ such that $T(X)$ is a compressed representation of $X$ (quantified by $I(X; T)$) that is most informative about $Y$ (quantified by $I(Y; T)$). Let $T$ be parameterized by $\theta$, the objective of IB optimizes the trade-off between $I(Y; T_{\theta}(X))$ and $I(X; T_{\theta}(X))$:
\begin{align}
\label{eq:IB}
    \min_{\theta} -I(Y, T_{\theta}(X)) + \beta I(X; T_{\theta}(X))
\end{align}
where $\beta$ is a positive Lagrange multiplier.

\citet{shwartz2017opening} casts finding of minimal sufficient statistics (MSS) $T(X)$ as a constrained optimization problem using data-processing inequality~\citep{cover1999elements}:
\begin{align}
\nonumber
    &\min_{T(X)}~~  I(T(X); X) \\
    &s.t~~ I(T(X); Y) = I(X;Y)
\end{align}
This corresponds to the IB method (Eq.~\ref{eq:IB}) which extends the notion of relevance between functions of samples and parameters in conventional MSS to any joint distribution of $X$ and $Y$. The IB method provides a computational framework for finding approximate MSS in a soft manner by trading off the sufficiency for $Y$ (I(Y; T(X))) and the minimality of the statistic ($I(X, T(X))$) with the Lagrange multiplier $\beta$~\citep{shwartz2017opening,shamir2010learning}.

\subsection{Connections between MLE and IB}
\label{app:mle&ib}
Given that the IB objective is approximately learning MSS in a soft manner, we next build the connections between the popularly adopted maximum likelihood estimation (MLE) in supervised learning and the IB objective. We show that under certain assumptions, MLE is approximating the IB objective defined on the joint distribution of $\ptrain(X, Y)$.

To facilitate the discussions, we decompose the model parameters into $\theta$ and $\phi$ that denote the parameters of the feature extractor $T_{\theta}(x)$ and the classifier respectively. MLE minimizes the expected negative log probability under $\ptrain(X, Y)$:
\begin{align}
    &\min_{\theta, \phi} \mathbb{E}_{x, y\sim \ptrain(X, Y)}[-\log p_{\theta, \phi}(x, y)] \\
    \iff & \min_{\theta, \phi} \mathbb{E}_{x, y\sim \ptrain(X, Y)}[-\log p_{\phi}(y|T_{\theta}(x)) - \log p_{\theta}(x)]
    \label{eq:mle0}
\end{align}
Usually, we only model the conditional distribution $p_{\phi}(Y|X)$ and assume that $p_{\theta}(X) = \ptrain(X)$ which is independent from $\theta$.
With the assumption that $p_{\theta}(x) \propto p^{\beta}(T_{\theta}(x))), \beta>0$, (\ref{eq:mle0}) can be rewritten as:
\begin{align}
    \min_{\theta, \phi} \mathbb{E}_{x, y\sim \ptrain(X, Y)}[-\log p_{\phi}(y|T_{\theta}(x))] + \beta \mathbb{E}_{x\sim \ptrain(X)}[-\log p(T_{\theta}(x))]
\label{eq:mle1}
\end{align}

Assume that the neural network parameterized by $\phi$ is a universal function approximator, then we can replace $\min_{\theta, \phi}$ with $\min_{\theta}$ and (\ref{eq:mle1}) can be written as: 
\begin{align}
     &\min_{\theta} H(Y|T_{\theta}(X)) + \beta H(T_{\theta}(X)) \\
    & \nonumber \mathrm{by\;(1)}\;I(Y; T_{\theta}(X)) = H(Y) - H(Y|T_{\theta}(X)) \\
    & \nonumber \mathrm{\;\;\;\;\;(2)}\; H(T_{\theta}(X)) = I(X; T_{\theta}(X)) + H(T_{\theta}(X)|X) = I(X; T_{\theta}(X))\\
    \iff & \min_{\theta} -I(Y; T_{\theta}(X)) + \beta I(X; T_{\theta}(X)) 
\label{eq:mle2}
\end{align}
We can see that under the assumption of $p_{\theta}(x) \propto p^{\beta}(T_{\theta}(x)))$, the MLE objective can be converted into the same form as the IB objective. In practice, we usually do not model $\ptrain(X)$ and only optimize the first term $I(Y; T_{\theta}(X))$ in (\ref{eq:mle2}). However, previous work~\citep{shwartz2017opening,geiger2020information} has shown that deep neural networks (DNNs) are implicitly minimizing $I(X; T_{\theta}(X))$ with a wide range of activation functions and architectures, which are manifested as a second compression phase during learning with SGD.
Thus, we can presumably consider MLE as approximating the IB objective, which is equivalent to learning the MSS on the train distribution $\ptrain(X, Y)$.

\section{Details of the Online Greedy Algorithm for Group DRO}
\label{app:alg:gdro}
\begingroup
\begin{center}
\removelatexerror% Nullify \@latex@error
\resizebox{!}{0.2 \columnwidth}{
\begin{algorithm}[H]
\SetAlgoLined
\DontPrintSemicolon
\SetKwInOut{Input}{Input}
\SetKwInOut{Output}{Output}
\SetCommentSty{newcomment}
\SetKwComment{Comment}{$\triangleright$\ }{}
\Input{$\alpha$; $m$: total number of groups}
Initialize historical average group losses $\hat{L}^{(0)}$; historical estimate of group probabilities $\hat{p}^{train(0)}$; learning rate $\eta$\\
\For{$t=1,\cdots,T$}{
    Sample a mini-batch batch $B=(\rvx, \rvy, \rvg)$ uniformly from $\ptrain$ \\
    \Comment{Update the historical vectors of $\hat{L}^{(t)}$ and $\hat{p}^{train(t)}$ for each group $g \in \{1, \cdots, m\}$}
    $\hat{L}^{(t)}(g) \leftarrow \mathrm{EMA}(\{ \ell(\rvx_i, \rvy_i; \theta^{(t-1)}): \rvg_i=g\}, ~ \hat{L}^{(t-1)}(g))$\\
    $\hat{p}^{train(t)} \leftarrow \mathrm{EMA}(\mathrm{\#samples~ of~ each~ group~ in~ B}, ~\hat{p}^{train(t-1)})$\\
    \Comment{Update the worst-case distribution $q^{(t)}$}
    Sort $\hat{p}^{train(t)}$ in the order of decreasing $\hat{L}^{(t)}$ and denote the sorted group indexes $\bm{\pi}$\\
    $q^{(t)}(g_{\bm{\pi}_i}) = \min\{\frac{\hat{p}^{train(t)}(g_{\bm{\pi}_i})}{\alpha}, 1 - \sum_{j=1}^{i-1}\frac{\hat{p}^{train(t)}(g_{\bm{\pi}_j})}{\alpha}\}$\\
    \Comment{Update model parameters $\theta$}
    $\theta^{(t)} = \theta^{(t-1)} - \frac{\eta}{|B|} \sum_{i=1}^{|B|}\frac{q^{(t)}(\rvg_i)}{\hat{p}^{train(t)}(\rvg_i)}\nabla\ell(\rvx_i, \rvy_i;\theta^{(t-1)})$
}
  \caption{\label{alg:gdro}Online greedy algorithm for group DRO~\citep{oren2019distributionally}}
\end{algorithm}
}
\end{center}
  \vspace{-3mm}
\endgroup
EMA refers to exponential weighted moving average such that $\mathrm{EMA}(v_1, v_2) = \gamma v_1 + (1-\gamma) v_2$, where $\gamma \in (0, 1)$.
%training data $\mathcal{D}_{train}:\{(x_i, y_i, g_i)\}$; 

\section{Synthetic Experiments: on Investigation Spurious Features under Covariate Shift}
\label{app:syn}
\begin{figure}[h]
    % \vspace{-2mm}
    \centering
    \includegraphics[width=0.35\textwidth]{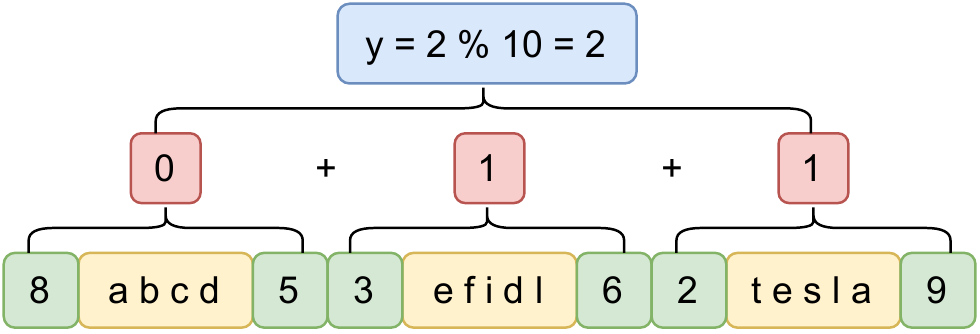}
    \vspace{-2mm}
    \caption{An illustrative example of the synthetic task.}
    \label{fig:syn}
    \vspace{-4mm}
\end{figure}

\paragraph{Synthetic Experiments}
We design synthetic experiments where data is generated based on the ground-truth rules and different biases are injected.
We show that even in the presence of necessary information to learn the rules, the ERM model (specifically, we examine MLE) can still learn spurious features or miss robust features under covariate shift.
The synthetic task aims to predict an integer $y \in \{0, \cdots, 9\}$ conditioned on a sequence $x$ as shown in Fig.~\ref{fig:syn}. 
Concretely, $x$ is composed of $m$ chunks, where each chunk $c_i$ has $|c_i|$ characters that are randomly sampled from an alphabet $\mathcal{V}$. 
We prepend an integer $c_i^1$ and append an integer $c_i^2$ to each chunk $c_i$, and both $c_i^1$ and $c_i^2$ are uniformly sampled from $[1, 10]$. 
The target integer $y$ is predicted following the rules: each triple of $(c_i^1, c_i, c_i^2)$ produces an indicator value $d_i$; $d_i = c_i^2 - c_i^1$ if $c_i^2 > c_i^1$, otherwise $d_i=0$; then $y = (\sum_{i=1}^m d_i) \mod 10$. We set $3 \leq m \leq 6$, $3 \leq |c_i| \leq 5$ and $|\mathcal{V}|=26$,. 
We use a one-layer bidirectional LSTM~\citep{hochreiter1997long} to model the input sequence and use the final hidden states of the LSTM to predict the target value. 
We create training data following the the above description and design two settings that introduce covariate shift to examine if the model can learn the rules with ERM.

\textbf{(a) Setting 1 --- ERM-trained models can miss robust features under covariate shift}:
% \gn{Notably, this is only models with different support, right? I wonder if there's anything interesting you could say/do about using the same support?}
% \cz{The following one also has different support, because the training data does not include examples that do not contain the special characters.}
% \gn{Yeah, my point was more that it would be interesting to do something where you have covariate shift under the same support (kinda like the experiments in \citet{warstadt-etal-2020-learning}, where you have a small number of examples that ambiguously demonstrate the true rule you want to capture.)}
We create the training data by imposing $c_m^2 > c_m^1$ on the last chunk $c_m$ of all the training samples. 
When we create the training data, the rules applied to each chunk are the same as described above, which means that the model does not need to learn additional rules for the last chunk. 
We are interested in examining whether the model trained with ERM will apply the rules learned from other chunks to the last one or it will miss the robust features of the last chunk.
At test time, we evaluate on two groups of test sets: $\mathcal{D}_{out}$ where $c_m^2 \leq c_m^1$, different from the training data, and $\mathcal{D}_{in}$ where $c_m^2 > c_m^1$, consistent with the training data. 
From Tab.~\ref{tab:syn}, we see that the test accuracy on $\mathcal{D}_{out}$ is much lower that that on $\mathcal{D}_{in}$. 
This demonstrates that the model only learns robust features from chunks $c_1^{m-1}$ but misses the robust features of the last chunk $c_m$. 
We conjecture that the model trained with ERM learns in a lazy way where it tries to minimize the entropy of learned features by memorizing patterns and taking shortcuts as discussed further in Appendix~\ref{app:mle&ib}.

\textbf{(b) Setting 2 --- ERM-trained models can learn spurious features under covariate shift}: 
In the second setting, we inject spurious patterns into the training data that co-occur with the rules we aim to learn. 
As both robust rules and spurious patterns co-exist in the training data, we would like to see whether the model picks up the spurious ones or the robust ones.
Specifically, each training input sequence has a chunk $c_j$ that includes a special segment of characters, e.g. \textsl{a b}. The remainder of $d_j=c_j^2-c_j^1$ and the sum of all indicators $\sum_{i=1}^m d_i$ mod by 10 are the same such that the target label $y$ is always the same as the indicator $d_j$. 
Similarly, we test on two cases: i) $\mathcal{D}_{in}$ where every sequence includes a special chunk as in the training set; ii) $\mathcal{D}_{out}$ where characters in each chunk are uniformly sampled. 
We can see from Tab.~\ref{tab:syn} that the model learns to use the spurious patterns to predict the target label instead of the general rules.

\begin{table}[t]
    \centering
    \small
    \begin{tabular}{l|cc}
    \toprule
    & $\mathcal{D}_{in}$ & $\mathcal{D}_{out}$ \\
    \midrule
    Setting 1   &  99.93 $\pm$ 0.02 & 14.68 $\pm$ 2.60 \\
    Setting 2   &  100.00 $\pm$ 0.00 & 10.26 $\pm$ 0.25 \\
    \bottomrule
    \end{tabular}
    \vspace{-2mm}
    \caption{Test accuracy of the synthetic task.}
    \label{tab:syn}
    \vspace{-5mm}
\end{table}

\vspace{-2mm}
\section{Experimental Details}
\label{app:exp}
\subsection{Models and Training Details} 
\paragraph{Model Specific Settings} 
In our method, we adopt two criterions in GC-DRO to determine when to update $q(x, y|g)$ for each groups: (1) update when the robust validation accuracy drops (2) update at every epoch. 
With (2), $q(x, y |g)$ is updated more frequently. 
For MNLI and Celeb-A, we use the second criterion. 
For FDCL18, we use the first criterion, because this is a relatively smaller dataset and updating $q(x, y|g)$ less frequently makes training more stable.
Every time $q(x, y|g)$ is updates, we clear the historical losses in EMA that is used for updating $q(g)$.
We use exponentially weighted moving average (EMA) to compute the historical losses for both $q(g)$ and $q(x, y|g)$, for which we denote $\mathrm{EMA}_{\text{G}}$ and and $\mathrm{EMA}_{\text{CG}}$ respectively.
As shown above, we use $\gamma$ to denote the coefficient for current value in EMA, thus $1-\gamma$ is used to the historical value.
We found that the value of $\gamma$ is an important hyperparameter in some cases to achieve better performance, since the final $q$ distribution is computed through sorting the losses accumulated via EMA.
Basically, a higher $\gamma$ pays more attention to the current value.
We search over $\{0.1, 0.5\}$ for both $\gamma$ used in $\mathrm{EMA}_{\text{G}}$ and  $\mathrm{EMA}_{\text{CG}}$ respectively. 
Through the robust accuracy on the validation set, we set both $\gamma$'s to be 0.5 for the NLP tasks except that for the imperfect partition of toxicity detection we set $\gamma$ used in $\mathrm{EMA}_{\text{G}}$ to be 0.1.
For the image task, we set both $\gamma$'s to be 0.1.
For the $\gamma$ used in accumulating the historical fractions of groups, we always use a small value 0.01.

\vspace{-3mm}
\paragraph{Training Details}
For the NLP tasks, we finetune a base Roberta model~\citep{liu2019roberta,ott2019fairseq} and we segment the input text into the sub-word tokens using the tokenization described in~\citep{liu2019roberta}.
During training, we sample minibatches that contain at most 4400 tokens. We train MNLI using Adam~\citep{kingma2014adam} with an intitial learnig rate of $1e-5$ for 35 epochs and FDCL18 for 45 epochs, and we linearly decay the learning rate at every step until the end of training. 
For the image task, we fine-tune a ResNet-18~\citep{he2016deep} for 50 epochs with batch size of 256. 
We use SGD with learning rate of $1e-4$.
At the end of every epoch, we evaluate the robust accuracy on the validation set.
We train on one Volta-16G GPU and it takes around 2 - 5 hours to finish one experiments for different datasets.

% \subsection{A Critical Bug in the Original Implementation of Greedy Group DRO}
\subsection{Implementation of the Group DRO Loss}
\label{app:bug}
We referred to the implementation of greedy group DRO in \citet{sagawa2019distributionally}, where they use the exact formulation in Eq.~\ref{eq:gdro} to compute the expected loss, which leads to inferior performance compared to the exponentiated-gradient based optimization as reported in \citet{sagawa2019distributionally}. 
The implementation computed the final loss by first computing the average loss over instances for each group (MC for the inner expectation), then compute the full expected value over the averaged group loss, as shown below:
\begin{equation}
\vspace{-3mm}
    \ell(\rvx, \rvy, \rvg; \theta) =  \sum_{g} q(g) \bar{\ell}(g) = 
    \sum_g q(g)\frac{1}{C_g}\sum_{\{i, \forall \rvg_i = g\}}{\ell(\rvx_i, \rvy_i; \theta)},
\end{equation}
where $(\rvx, \rvy, \rvg)$ is a mini-batch and $C_g$ is the number of samples that belong to group $g$ in the mini-batch. We can see that instances that belong to different groups are weighted correspondingly by the number of group size in a mini-batch. This causes that instances in large group get unfairly lower weights, especially when its probability in the $q$ distribution is low.
We fix this by directly computing the expected loss over the joint distribution of $q(x, y, g)$, i.e.~$\mathbb{E}_{(x_i,y_i,g_i)\sim q(x,y,g)} \ell(x_i, y_i, g_i; \theta) = \mathbb{E}_{(x_i, y_i, g_i)\sim \ptrain(x, y, g)} \frac{q(x_i, y_i, g_i)}{\ptrain(x_i, y_i, g_i)}\ell(x_i, y_i, g_i)$. 
Specifically, we do this by summing over all the importance weighted instance losses using corresponding group weights and taking average.
This allows us to obtain unbiased gradient estimates of $\theta$.
\begin{equation}
    \frac{1}{N} \sum_i \frac{q(\rvg_i)}{\ptrain(g_i)}\ell(\rvx_i, \rvy_i; \theta)
\end{equation}

% \begin{figure}[h]
%     \vspace{-2mm}
%     \centering
%     \includegraphics[width=0.45\textwidth]{figs/plot_hier_heatmap.pdf}
%     \vspace{-6mm}
%     \caption{The heatmap of summarized learning weights for different groups from group DRO and our method.}
%     \label{fig:heatmap}
%     \vspace{-5mm}
% \end{figure}

% \begin{figure}[h]
%     \vspace{-2mm}
%     \centering
%     \includegraphics[width=0.5\textwidth]{figs/ablation.pdf}
%     \vspace{-6mm}
%     \caption{Ablation studies on $\alpha$ and $\beta$ on the MNLI datasets.}
%     \label{fig:ablation}
%     \vspace{-4mm}
% \end{figure}

% \begin{table}[t]
%     \centering
%     \small
%     % \setlength{\tabcolsep}{3pt}
%     \begin{tabular}{l|cc|cc}
%         \toprule
%          \multirow{2}{*}{} & \multicolumn{2}{c|}{$G_1$} &  \multicolumn{2}{c}{$G_2$}\\
%         & $S=\text{desert}$ & $S=\text{greenlands}$ & $S=\text{desert}$ & $S=\text{greenlands}$ \\
%         \midrule
%         $P(Y=\text{camel}|S)$ & 0.5 & 0 & 1 & 0.5 \\
%         $P(Y=\text{cow}~~~|S)$ & 0.5 & 1 & 0 & 0.5 \\
%         \bottomrule
%     \end{tabular}
%     \vspace{-2mm}
%     \caption{An example of imperfect partition.}
%     \label{tab:gdro:example}
%     \vspace{-5mm}
% \end{table} 
%%%%%%%%%%%%%%%%%%%%%%%%%%%%%%%%%%%%%%%%%%%%%%%%%%%%%%%%%%%%%%%%%%%%%%%%%%%%%%%
%%%%%%%%%%%%%%%%%%%%%%%%%%%%%%%%%%%%%%%%%%%%%%%%%%%%%%%%%%%%%%%%%%%%%%%%%%%%%%%

\end{document}